\documentclass{article}

\usepackage{hyperref}
\usepackage{url}
\usepackage{booktabs}
\usepackage{graphicx}
\usepackage{amsthm}
\usepackage{tabularray}
\usepackage{cancel}
\usepackage{amsmath} 
\usepackage{mathtools}
\usepackage{array}
\usepackage{multirow}
\usepackage{amssymb}
\usepackage[T1]{pbsi}
\usepackage{subcaption}
\usepackage{bbold}

\usepackage{pifont}
\usepackage[most]{tcolorbox}

\usepackage{pgfplots}
\usepackage{tikz}
\pgfplotsset{compat=1.18}

\usetikzlibrary{patterns}
\usetikzlibrary{decorations.pathreplacing}
\usepgfplotslibrary{fillbetween}
\usetikzlibrary{calc}
\usetikzlibrary{patterns}
\usetikzlibrary{patterns.meta}

\usepackage[most]{tcolorbox}
\usepackage{color,soul}

\newtcbox{\transpar}{blank, on line, opacitytext=0.3}

\makeatletter
\NewDocumentCommand{\brushfrac}{mm}{%
  \mbox{%
    \bsifamily
    \check@mathfonts
    \sbox\z@{/}%
    \raisebox{\dimeval{\ht\z@-\height}}{\fontsize{\sf@size}{0}\selectfont#1}%
    \kern-0.2em/\kern-0.2em
    \raisebox{-\dp\z@}{\fontsize{\sf@size}{0}\selectfont#2}%
  }%
}
\makeatother
\newcommand{\cmark}{\ding{51}}%
\newcommand{\xmark}{\transpar{\ding{55}}}%
\NewDocumentCommand{\cxmark}{}{\transpar{\brushfrac{1}{2}}}

\newcommand{\arealeft}{\protect\tikz[baseline=0.5pt]{\protect\draw[
thick,
red,
preaction={fill=red,opacity=.05},
draw opacity=1,
pattern={Lines[distance=1mm,angle=135,line width=0.1mm]},
pattern color=red!50!white]   rectangle(3ex, 1.7ex);} \!\,}

\newcommand{\arearight}{\protect\tikz[baseline=0.5pt]{\protect\draw[
thick,
red,
preaction={fill=red,opacity=.05},
draw opacity=1,
pattern={Lines[distance=1mm,angle=-135,line width=0.1mm]},
pattern color=red!50!white]   rectangle(3ex, 1.7ex);} \!\,}

\newtheorem{innercustomgeneric}{\customgenericname}
\providecommand{\customgenericname}{}
\newcommand{\newcustomtheorem}[2]{%
  \newenvironment{#1}[1]
  {%
   \renewcommand\customgenericname{#2}%
   \renewcommand\theinnercustomgeneric{##1}%
   \innercustomgeneric
  }
  {\endinnercustomgeneric}
}

\newcustomtheorem{customthm}{Theorem}
\newcustomtheorem{customprop}{Proposition}

\def\lognormalpdf(#1, #2, #3){(1/(#1*#3*sqrt(2 * pi))) * exp(-((ln(#1) - #2)^2)/(2 * #3^2))}

\newcommand*\circled[1]{\tikz[baseline=(char.base)]{
            \node[shape=circle,draw,inner sep=2pt] (char) {#1};}}

\DeclareMathOperator*{\argmin}{arg\,min}

\usepackage{hyperref}

\usepackage[accepted]{icml2024}

\usepackage[capitalize,noabbrev]{cleveref}

\theoremstyle{plain}
\newtheorem{theorem}{Theorem}[section]

\theoremstyle{definition}
\newtheorem{definition}[theorem]{Definition}

\theoremstyle{remark}

\usepackage[textsize=tiny]{todonotes}

\usepackage{thm-restate}

\makeatletter
\newcommand{\IfRestatedTF}[2]{\ifthmt@thisistheone #2\else #1\fi}
\makeatother

\icmltitlerunning{Relaxed Quantile Regression: Prediction Intervals for Asymmetric Noise}

\begin{document}

\twocolumn[
\icmltitle{Relaxed Quantile Regression: Prediction Intervals for Asymmetric Noise}



\icmlsetsymbol{equal}{*}

\begin{icmlauthorlist}
\icmlauthor{Thomas Pouplin}{equal,yyy}
\icmlauthor{Alan Jeffares}{equal,yyy}
\icmlauthor{Nabeel Seedat}{yyy}
\icmlauthor{Mihaela van der Schaar}{yyy}
\end{icmlauthorlist}

\icmlaffiliation{yyy}{Department of Applied Mathematics and Theoretical Physics, University of Cambridge, UK}

\icmlcorrespondingauthor{Thomas Pouplin}{tp531@cam.ac.uk}

\icmlkeywords{Machine Learning, ICML, Quantile regression, interval regression, pinball loss, neural networks}

\vskip 0.3in
]



\printAffiliationsAndNotice{\icmlEqualContribution} 

\begin{abstract}
{Constructing valid prediction intervals rather than point estimates is a well-established approach for uncertainty quantification in the regression setting. Models equipped with this capacity output an interval of values in which the ground truth target will fall with some prespecified probability. This is an essential requirement in many real-world applications where simple point predictions' inability to convey the magnitude and frequency of errors renders them insufficient for high-stakes decisions. Quantile regression is a leading approach for obtaining such intervals via the empirical estimation of quantiles in the (non-parametric) distribution of outputs. This method is simple, computationally inexpensive, interpretable, assumption-free, and effective. However, it does require that the specific quantiles being learned are chosen a priori. This results in (a) intervals that are arbitrarily symmetric around the median which is sub-optimal for realistic skewed distributions, or (b) learning an excessive number of intervals. In this work, we propose Relaxed Quantile Regression (RQR), a direct alternative to quantile regression based interval construction that removes this arbitrary constraint whilst maintaining its strengths. We demonstrate that this added flexibility results in intervals with an improvement in desirable qualities (e.g. mean width) whilst retaining the essential coverage guarantees of quantile regression.}
\end{abstract}

\section{Introduction}

Reliable uncertainty estimation is an essential requirement for safely and robustly deploying neural networks in real-world applications \citep{amodei2016concrete, dietterich2017steps, kompa2021second}. However, research has consistently shown this to be a challenging problem in practice \citep{guo2017calibration, yao2019quality, ayhan2022test}. Therefore, significant efforts have been made to address this task in order to contribute towards more reliable and trustworthy models (see e.g. \citet{gawlikowski2023survey}). A significant aspect of this effort is developing regression methods that output predictive intervals rather than point predictions. This has proven to be a crucial requirement in high-stakes applications including medical decision-making \citep{begoli2019need}, autonomous driving \citep{su2023uncertainty}, and energy forecasting \citep{wang2022deep}. 

Especially in the case of neural networks, \textit{quantile regression} \citep{koenker1978regression} has emerged as a powerful method for obtaining such intervals. This approach requires the model to output estimates of two quantiles rather than a single point prediction, which is easily optimized in practice by a simple change in loss function. These quantiles may then be used to construct an interval $(\mu_1, \mu_2)$ within which the true label will lie with probability $\alpha$ (a formal description of quantile regression is provided in \Cref{sec:background}). Obtaining predictive intervals via quantile regression has earned substantial popularity in both research and practice \citep{koenker2001quantile, koenker2017quantile, yu2003quantile, fitzenberger2001economic}. This uptake can be attributed to several factors, including (a) methodological simplicity requiring minimal changes to the modeling procedure, (b) negligible increased computational cost (in contrast to e.g. ensemble methods \citep{lakshminarayanan2017simple}), (c) a simple, easily interpreted characterization of uncertainty \citep{savelli2013advantages, goodwin2010forecasts}, (d) lack of parametric assumptions on the data-generating process, and (e) enduring empirical effectiveness \citep{chung2021beyond, tagasovska2019single}. Furthermore, quantile regression methods can be wrapped in the conformal prediction procedure of \citet{vovk2005algorithmic} to additionally provide finite sample coverage guarantees, as demonstrated in \citet{romano2019conformalized}.

This work primarily focuses on the standard task of outputting predictive intervals that obtain a \textit{prespecified} level of coverage (but the method could also be applied to the related non-prespecified task e.g. \citet{chung2021beyond}). This problem does not have a unique solution (consider e.g. trivial solutions in which some percentage of predictions are given infinite width intervals). Therefore, several variations of quantile regression have been introduced in recent years (see \mbox{\Cref{sec:background}}) which are typically also evaluated based on \emph{additional} desirable properties of their resulting prediction intervals. These include minimizing interval width and achieving improved conditional coverage (see e.g. \mbox{\citet{pearce2018high,feldman2021improving}} respectively, and further discussion in \mbox{\Cref{sec:background}}).

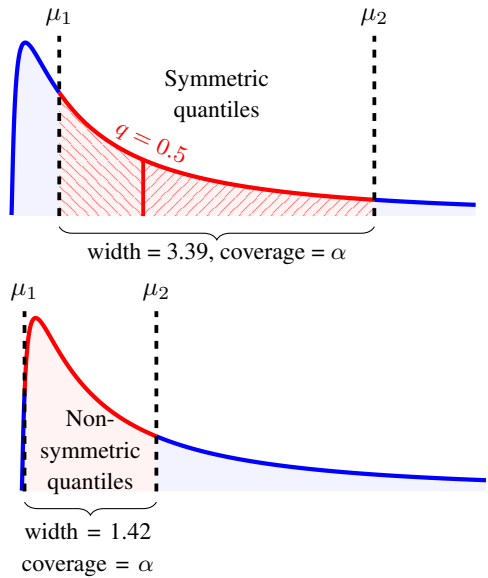
\begin{figure}
\centering
\begin{tikzpicture}

\def\intervalwidth{3.39}   
\def\lowerquantile{0.5151459}   
\def\upperquantile{3.9090919}   
\def\lnmu{0.35}  
\def\lnsigma{1.5}   
\def\visxstart{-0.1}    
\def\visxend{5.1}    
\def\visystart{-0.2}    
\def\startx{-0.2}   
\def\endx{5}    
\def\textheight{0.4}
\def\medianx{1.41907}

\def\filllinedist{1mm} 

\begin{axis}[
hide axis,
xmin=\visxstart,
xmax=\visxend,
ymin=\visystart,
samples=200,
height=5cm,
width=8cm,
clip=false
]

\path[name path=baseLHS] (axis cs:0,0) -- (axis cs:\lowerquantile,0);
\addplot[name path=lognormalpathLHS, domain=\startx:\lowerquantile, blue, ultra thick] {\lognormalpdf(x, \lnmu, \lnsigma)};
\addplot [
        thick,
        color=blue,
        fill=blue, 
        fill opacity=0.05
    ]
fill between[
        of=lognormalpathLHS and baseLHS,
    ];

\path[name path=baseMID] (axis cs:\lowerquantile,0) -- (axis cs:\upperquantile,0);
\addplot[name path=lognormalpathMID, domain=\lowerquantile:\upperquantile, red, ultra thick] {\lognormalpdf(x, \lnmu, \lnsigma)};
\addplot [
        thick,
        color=blue,
        fill=red, 
        fill opacity=0.05
    ]
fill between[
        of=lognormalpathMID and baseMID,
    ];

\path[name path=baseRHS] (axis cs:\upperquantile,0) -- (axis cs:\endx,0);
\addplot[name path=lognormalpathRHS, domain=\upperquantile:\endx, blue, ultra thick] {\lognormalpdf(x, \lnmu, \lnsigma)};
\addplot [
        thick,
        color=blue,
        fill=blue, 
        fill opacity=0.05
    ]
fill between[
        of=lognormalpathRHS and baseRHS,
    ];

\addplot[draw=none, pattern={Lines[
                  distance=\filllinedist,
                  angle=135,
                  line width=0.1mm]},pattern color=red!50!white]
fill between[
    of=lognormalpathMID and baseMID,
    soft clip={domain=\lowerquantile:\medianx},
];
\addplot[draw=none, pattern={Lines[
                  distance=\filllinedist,
                  angle=-135,
                  line width=0.1mm]},pattern color=red!50!white]
fill between[
    of=lognormalpathMID and baseMID,
    soft clip={domain=\medianx:\upperquantile},
];
\draw [color=red, line width=0.5mm] (\medianx,0) -- (\medianx,{\lognormalpdf(\medianx, \lnmu, \lnsigma)}) node [above, rotate=-25, yshift=-0.2ex] {\small $q = 0.5$};

\draw [dashed, line width=0.5mm] (\lowerquantile,0) -- (\lowerquantile,0.6) node [above] {$\mu_1$};
\draw [dashed, line width=0.5mm] (\upperquantile,0) -- (\upperquantile,0.6) node [above] {$\mu_2$};
\draw [decorate,decoration={brace, mirror, amplitude=5pt},yshift=-4pt] (\lowerquantile,0) -- (\upperquantile, 0) node [black,midway,yshift=-10pt] {\small width = \intervalwidth, coverage = $\alpha$};

\end{axis}
\end{tikzpicture}
\qquad
\begin{tikzpicture}
\def\intervalwidth{1.42}   
\def\lowerquantile{0.03}   
\def\upperquantile{1.45}   
\def\lnmu{0.35}  
\def\lnsigma{1.5}   
\def\visxstart{-0.1}    
\def\visxend{5.1}    
\def\visystart{-0.2}    
\def\startx{-0.2}   
\def\endx{5}    
\def\textheight{0.13}

\begin{axis}[
hide axis,
xmin=\visxstart,
xmax=\visxend,
ymin=\visystart,
samples=200,
height=5cm,
width=8cm,
clip=false
]

\path[name path=baseLHS] (axis cs:0,0) -- (axis cs:\lowerquantile,0);
\addplot[name path=lognormalpathLHS, domain=\startx:\lowerquantile, blue, ultra thick] {\lognormalpdf(x, \lnmu, \lnsigma)};
\addplot [
        thick,
        color=blue,
        fill=blue, 
        fill opacity=0.05
    ]
fill between[
        of=lognormalpathLHS and baseLHS,
    ];

\path[name path=baseMID] (axis cs:\lowerquantile,0) -- (axis cs:\upperquantile,0);
\addplot[name path=lognormalpathMID, domain=\lowerquantile:\upperquantile, red, ultra thick] {\lognormalpdf(x, \lnmu, \lnsigma)};
\addplot [
        thick,
        color=blue,
        fill=red, 
        fill opacity=0.05
    ]
fill between[
        of=lognormalpathMID and baseMID,
    ];

\path[name path=baseRHS] (axis cs:\upperquantile,0) -- (axis cs:\endx,0);
\addplot[name path=lognormalpathRHS, domain=\upperquantile:\endx, blue, ultra thick] {\lognormalpdf(x, \lnmu, \lnsigma)};
\addplot [
        thick,
        color=blue,
        fill=blue, 
        fill opacity=0.05
    ]
fill between[
        of=lognormalpathRHS and baseRHS,
    ];

\draw [dashed, line width=0.5mm] (\lowerquantile,0) -- (\lowerquantile,0.6) node [above] {$\mu_1$};
\draw [dashed, line width=0.5mm] (\upperquantile,0) -- (\upperquantile,0.6) node [above] {$\mu_2$};
\draw [decorate,decoration={brace, mirror, amplitude=5pt},yshift=-4pt] (\lowerquantile,0) -- (\upperquantile, 0) node [black,midway,yshift=-18pt,xshift=-1pt,text width=2.3cm, align=center] {\small width = \intervalwidth \; coverage = $\alpha$};

\end{axis}
\end{tikzpicture}
\vspace{-0.25cm}
\caption{\small \textbf{Symmetric quantiles.} We compare two pairs of intervals on an identical (non-symmetric) log-normal probability distribution where in both cases a fixed level of coverage $\alpha$ is obtained. In the upper figure, the intervals are selected to be symmetric in terms of probability mass around the median $q = 0.5$ (i.e. the two lined regions \arealeft and \arearight contain equal probability mass) as in the case of quantile regression. In the lower figure, we remove this constraint and obtain a much narrower interval with identical coverage. }
\label{fig:skewillustration}
\end{figure}

In this work our \textbf{contributions} are threefold: \textbf{(1)} In \Cref{sec:limitations} we identify a substantial inefficiency in the standard procedure of first estimating quantiles from which predictive intervals are then derived. We show that this typically results in intervals with their midpoint fixed at the median which is undesirable for non-symmetric distributions (\Cref{fig:skewillustration}) \textit{or} requires learning more quantiles than necessary resulting in a more complex learning problem and, therefore, sub-optimal performance (see e.g. SQR in \Cref{sec:background}); \textbf{(2)} In \Cref{sec:proposedresolution} we propose a novel objective which directly learns intervals \textit{without} a priori specifying particular quantiles. We then equip this function with a regularization term that aids in selecting among possible interval choices by rewarding desirable properties such as narrower intervals or improved conditional coverage. This results in a method we term Relaxed Quantile Regression (RQR); \textbf{(3)} We theoretically show that the solution of our proposed objective achieves valid coverage in expectation with bounded variance (\Cref{sec:proposedresolution}). Empirically, we find that it results in superior performance to existing methods when evaluated on standard benchmarks (\Cref{sec:experiments}).

\section{Background} \label{sec:background}
In this section, we provide a summary of relevant existing works that convert neural networks from outputting point estimates to outputting predictive intervals (i.e. \textit{single model approaches}). In \Cref{app:extendedrelatedwork} we provide a broad summary of predictive interval generation more generally and highlight some of the unique advantages of the single model approach that we consider in this text.

\textbf{Deriving intervals from quantiles.} Throughout this work we consider the standard regression task consisting of input/target pairs $(\mathbf{X},Y) \in \mathbb{R}^d \times \mathbb{R}$ with $d \in \mathbb{N}$ where bold denotes vectors and non-bold denotes scalars. We express realizations of these random variables (i.e. data) using lower-case $(\mathbf{x},y)$. Denoting the cumulative distribution function of a probability distribution with $\mathbb{F}$, we recall that a \textit{quantile function} is given by $\mathbb{F}^{-1}(p) = \inf\{q \in \mathbb{R}: p \leq \mathbb{F}(q)\}$ where $p \in (0,1)$ is the desired \textit{quantile probability}. Specifically, this provides some \textit{quantile value} $q$ such that $\mathbb{P}(Y \leq q) = p$. In the machine learning setting, we are generally interested in the probability distribution of $Y$ \textit{conditional} on a given input $\mathbf{X}=\mathbf{x}$. Throughout this work, we refer to the task of estimating a quantile value $q$ corresponding to a particular quantile probability $p$ from data as \textit{estimating quantiles}. Once we have some function $\mu: \mathbb{R}^d \to \mathbb{R}$ for estimating quantiles (e.g. a neural network), we might wish to construct a conditional \textit{interval} such that $\mathbb{P}(\mu_1(\mathbf{x}) \leq Y \leq \mu_2(\mathbf{x})) = \alpha$ with $\mu_1(\mathbf{x}) < \mu_2(\mathbf{x})$\footnote{As a minor technical note, when a method is \textit{permutation invariant} (as with our proposed method described in \mbox{\Cref{sec:proposedresolution}}), this expression becomes $\mathbb{P}(\min(\mu_1(\mathbf{x}),\mu_2(\mathbf{x})) \leq Y \leq \max(\mu_1(\mathbf{x}),\mu_2(\mathbf{x})) ) = \alpha$ and thus no longer assumes one bound to be greater than the other.} and $\alpha \in (0,1)$. In other words, a pair of bounds between which the target will lie with some desired probability $\alpha$. We will generally drop the dependence of $\mu_1$ and $\mu_2$ on $\mathbf{x}$ for ease of notation. Clearly, this interval can be easily derived from the quantile function by simply noting that $\mathbb{P}(\mu_1 \leq Y \leq \mu_2) = \mathbb{P}(Y \leq \mu_2) - \mathbb{P}(Y \leq \mu_1)$. Therefore the problem of constructing valid intervals may be solved by approximating the quantile function to estimate the appropriate quantiles \textit{and then} constructing an interval from these quantile values. However, this assumes that we know which quantile probabilities we should use in advance as any pair of quantile probabilities $p_l < p_u$ such that $p_u - p_l = \alpha$ will result in an interval with $\alpha$ coverage. Therefore, existing approaches typically select symmetric quantile probabilities or estimate an infinite number of quantile probabilities from which intervals can be constructed. We discuss the consequences of this fact in \Cref{sec:limitations}.

\begin{table}
\caption{\small \textbf{Quantile regression methods.} Several desirable properties of these methods are considered. Note that \cxmark\!\: denotes partially satisfying a property as the SQR \& IR objectives can trade-off between coverage and interval width but do not directly consider conditional coverage. }
\label{tab:dantable}
\centering
\resizebox{\columnwidth}{!}{%
\begin{tabular}{c|ccccc} \toprule
    Property & QR & SQR & IR & Ours \\ \midrule
    \multirow{2}{*}{\shortstack{Suitable for non-centered\\ distributions}}  & \multirow{2}{*}{\xmark} & \multirow{2}{*}{\cmark} & \multirow{2}{*}{\cmark} & \multirow{2}{*}{\cmark} \\
    &&&& \\
    \multirow{2}{*}{\shortstack{Avoids explicitly learning all\\quantiles}}  & \multirow{2}{*}{\cmark}  & \multirow{2}{*}{\xmark}  &  \multirow{2}{*}{\cmark}  & \multirow{2}{*}{\cmark}  \\
    &&&& \\
    \multirow{2}{*}{\shortstack{Dynamically controls the trade-off \\ between any desirable objectives}}  & \multirow{2}{*}{\xmark}  & \multirow{2}{*}{\cxmark} & \multirow{2}{*}{\cxmark}  & \multirow{2}{*}{\cmark}  \\
    &&&& \\
    \multirow{2}{*}{Asymptotic coverage guarantees} & \multirow{2}{*}{\cmark}  & \multirow{2}{*}{\cmark}  &  \multirow{2}{*}{\xmark} & \multirow{2}{*}{\cmark}  \\
    &&&& \\
    \multirow{2}{*}{\shortstack{No Gaussian approx. or assumption \\of iid miscoverage of instances}} & \multirow{2}{*}{\cmark}  & \multirow{2}{*}{\cmark}  & \multirow{2}{*}{\xmark} & \multirow{2}{*}{\cmark}  \\
    &&&& \\ \bottomrule
\end{tabular}
}
\end{table}

\textbf{Quantile Regression (QR).} Here we refer to such methods that aim to provide intervals by accurately estimating quantiles. In the case of neural networks, the key distinction from standard point estimation methods which predict the expected value $\mathbb{E}(Y|\mathbf{X})$ lies in the choice of loss function. We require a loss function $\mathcal{L}: \cdot \to \mathbb{R}^+$ mapping from a quantile estimate to a scalar loss upon which we can apply gradient descent. Perhaps the most widely known approach is that of the \textit{pinball loss} function (also known as \textit{quantile loss}) of \citet{koenker1978regression, steinwart2011estimating}. For a quantile estimator $\mu: \mathbb{R}^d \to \mathbb{R}$, the pinball loss expression\footnote{We note that the Winkler Score (or interval score) objective \cite{dunsmore1968bayesian,winkler1972decision}, which also regularly appears in the literature, is proportional to the QR objective as shown in Appendix \ref{app:proofs} and learns all quantiles like SQR.}
\begin{equation*}
      \rho_{q}(\mu, \mathbf{x}, y) = 
    \begin{cases}
      q(y -\mu(\mathbf{x})) & \text{if $y -\mu(\mathbf{x}) \geq 0$}\\
      (q-1) (y -\mu(\mathbf{x})) & \text{if $y -\mu(\mathbf{x}) < 0$ }.\\
    \end{cases}       
\end{equation*}
{Then a strategy to construct an interval of targeted coverage level of $\alpha$ involves estimating two specific conditional quantiles, denoted as $q_l$ and $q_u$, where $q_l$ corresponds to the $\frac{1-\alpha}{2}$ quantile, and $q_u$ corresponds to the $1\!-\!\frac{1-\alpha}{2}$ quantile. Thus, the loss function optimized by the neural network is given by}
\begin{equation*}
\mathcal{L}^{\text{QR}}_\alpha((\mu_1, \mu_2), \mathbf{x}, y) = \rho_{\frac{1-\alpha}{2}}(\mu_1, \mathbf{x}, y) + \rho_{1-\frac{1-\alpha}{2}}(\mu_2, \mathbf{x}, y).    
\end{equation*}
This methodology ensures that the probability of the ground truth target $y$ falling within the interval $[{\mu_1},\mu_2]$ is $\alpha$, thereby establishing the desired mean coverage. In practice, two particular quantiles are typically predefined and learned using a single neural network with two outputs. We refer to this antecedent approach as \textit{quantile regression (QR)} throughout this work. 

\textbf{Simultaneous Quantile Regression (SQR).} Rather than predefining two particular quantiles, \citet{tagasovska2019single} propose to learn \textit{all} possible quantiles with a single output model by augmenting the neural network with an additional input for the desired quantile. We express this simultaneous quantile regressor as $\mu_q(\mathbf{x})$. Throughout training the quantile $q$ is stochastically selected from a uniform distribution where any quantile loss function may be applied (e.g. pinball).

A model trained using SQR is underspecified in the sense that there are infinite quantile pairs that can be used to construct a valid interval of a given coverage level at test time. Therefore we consider two strategies for selecting a particular interval: (a) \textit{SQR-C} selects the centered interval (as in QR) assuming that jointly learning all quantiles will produce a more accurate estimator; (b) \textit{SQR-N} selects the pair of (potentially non-centered) quantiles that produce the narrowest interval.

\textbf{Interval Regression (IR).} Learning the intervals directly without the intermediate step of first learning quantiles is an alternative approach that has emerged somewhat independently of the quantile regression literature. A method proposed in \citet{pearce2018high} achieves best-in-class empirical performance by introducing a loss function that attempts to balance coverage with interval width. By making the strong assumptions that (a) the cases of miscoverage are \emph{iid} and (b) batch sizes are sufficiently large for the binomial distribution to be well approximated by a Gaussian, the authors derive the following objective
\begin{multline*}
    \mathcal{L}^{\text{IR}}_\alpha((\mu_1, \mu_2), \mathbf{x}, y) = \frac{n}{c}(\mu_2- \mu_1) \cdot \mathbb{I}_{\mu_1\leq y \leq \mu_2}\\ + \lambda \frac{n}{\alpha (1 - \alpha)} \max(0,  \alpha - \frac{c}{n}) \\
\end{multline*}
Here $\lambda \in \mathbb{R}$ denotes a hyperparameter weighting term and $n$ denotes the batch size, and $c \vcentcolon= \sum_{i=1}^n \mathbb{I}_{\mu_1\leq y \leq \mu_2}$. Unfortunately, unlike quantile regression methods (and our proposed method later in this work), it is unclear if this method achieves theoretical coverage guarantees.

\textbf{Evaluation.} Competing approaches to obtaining prediction intervals are typically compared across a range of well-established desirable properties. As we only have access to a finite dataset in practice, asymptotic coverage is not guaranteed resulting in the need to evaluate \textit{calibration}. Ideally, we would produce intervals that accurately model the \textit{conditional} probability $P(\mu_1(\mathbf{x}) < y < \mu_2(\mathbf{x}) | \mathbf{x})$. Unfortunately, in the standard setting, we cannot directly estimate this quantity \citep{zhao2020individual} and instead often consider the \textit{marginal} probability $P(\mu_1(\mathbf{x}) < y < \mu_2(\mathbf{x}))$. We can estimate the marginal calibration on the test set using the Prediction Interval Coverage Probability (PICP) which simply measures the ratio of observations falling inside their intervals \citep{kuleshov2018accurate, tagasovska2019single}. 

Despite the aforementioned impossibility of exactly estimating the former conditional quantity, several proxy metrics have been proposed that test for independence between examples $\mathbf{x}$ and instances of miscoverage. These include using Pearson's correlation between interval width and miscoverage cases \citep{feldman2021improving} and the independence rewarding Hilbert-Schmidt independence criterion (HSIC) \citep{greenfeld2020robust}.

However, probably the most common criteria of evaluation considers aggregate \textit{interval width}. For a fixed level of coverage, narrower intervals are typically considered preferable. This can also prevent trivial solutions with some potentially infinite width intervals, which may still satisfy empirical tests of coverage. Sometimes referred to as \textit{sharpness} \citep{gneiting2007probabilistic} or \textit{adaptive coverage} \citep{seedat2023improving}, we primarily consider Mean Prediction Interval Width (MPIW) which measures the mean interval width across the test data \citep{tagasovska2019single}. We provide (a) some extended discussion on the motivation for these objectives in \Cref{sec:intervalwidth} and (b) a formal description of all evaluation metrics in \Cref{app:metrics} to ensure that this work is self-contained.  

\section{Relaxed Quantile Regression}
\subsection{Highlighting the Limitation of Existing Quantile Regression Methods} \label{sec:limitations}
Using quantile estimation as a means for obtaining predictive intervals may be viewed as a victim of Vapnik's famous heuristic that ``when solving a problem of interest, do not solve a more general problem as an intermediate step'' \citep{vapnik2006estimation}. The fundamental limitation of estimating predictive quantiles as an intermediate step toward estimating predictive intervals becomes apparent by closely considering the quantile regression approach. 

The standard quantile regression approach consists of selecting the two quantiles a priori such that the region between them results in a predictive interval with a desired level of coverage. In practice, for a desired coverage level $\alpha$, the standard approach is to select the $\frac{1-\alpha}{2}$ and $1-\frac{1-\alpha}{2}$ quantiles. However, selecting a specific pair of \textit{non-symmetric} quantiles would require knowledge of the underlying noise distribution which is unknown. As illustrated in \Cref{fig:skewillustration}, when the underlying noise distribution around $Y$ is, in fact, non-symmetric this results in wider than necessary intervals due to being arbitrarily centered (in terms of probability mass) around the median\footnote{Of course, if centered intervals are \textit{required} then quantile regression is appropriate. However, this is unlikely to be a wise requirement in the prevalent case of a skewed target distribution.}. On real-world tasks, we should expect non-symmetric noise distributions to be ubiquitous. This has been highlighted in previous work \citep{tagasovska2019single} and we further illustrate this by including histograms of the target distributions of popular, real-world datasets in \Cref{tab:widthresultsfull} and their summary statistics in \Cref{tab:summarystats}. Whilst the true noise distribution cannot be known on real data, the shapes of these empirical distributions suggest that perfect symmetry is a very strong assumption to hold over natural phenomena.

The existing resolution to this issue, as introduced by \citet{tagasovska2019single}, is to learn \textit{all} possible quantile probabilities in $(0,1)$ (see e.g. SQR in \Cref{sec:background}). This is explicitly aimed at rectifying the aforementioned limitation as the authors note that it enables them to ``model non-Gaussian, skewed, asymmetric, multimodal, and heteroskedastic aleatoric noise in data''. The idea being that, once all quantiles are learned, any pair may be selected such that they satisfy $\alpha$ coverage in addition to other qualities (e.g. narrower intervals). However, this introduces a significantly more challenging learning problem of estimating an infinite number of quantiles rather than exactly two for each example which can negatively impact the performance of the underlying quantile estimator. Furthermore, it is not obvious how to select a specific interval at test time as simply selecting the narrowest valid interval is likely to induce a bias that can negatively impact empirical coverage. In our experiments in \Cref{sec:experiments}, we show that these drawbacks result in this approach generally achieving inferior performance when compared against the former approach despite its added flexibility.

In \Cref{sec:proposedresolution} we introduce an alternative approach that solves the interval estimation problem whilst removing the median-centering constraint. This is achieved by relaxing the requirement to select symmetric quantiles a priori from which intervals can be constructed and, instead, estimating a pair of potentially asymmetric intervals \emph{directly}. Although the previous work of \citet{pearce2018high} also considers a direct interval regression approach, ours is the first to build upon the quantile regression literature, thereby maintaining a strong theoretical foundation and converging to a solution that achieves coverage guarantees in expectation. This is reflected in superior coverage when empirically evaluated in \Cref{sec:experiments}. \Cref{tab:dantable} summarizes the key differences between our proposed method and those of previous works.

\subsection{Proposed Resolution: Relaxed Quantile Regression} \label{sec:proposedresolution}
\textbf{Relaxed Quantile Regression (RQR).} We begin by introducing a novel objective which \textit{directly} learns intervals without the intermediate step of prespecifying quantiles (i.e. \textit{relaxing} the quantile learning requirement). For a targeted coverage level $\alpha$ and a neural network outputting two interval bounds $(\mu_1,\mu_2)$ we minimize
\begin{align}
    \mathcal{L}^{\text{RQR}}_{\alpha}((\mu_1,\mu_2), \mathbf{x}, y) = 
    \begin{cases}
      \alpha \kappa & \text{if $\kappa \geq 0$}\\
      (\alpha-1) \kappa& \text{if $\kappa < 0$ }\label{eqn:RQRbaseobjective}\\
    \end{cases} \\     
    \text{with } \kappa = (y - \mu_1)(y - \mu_2). \notag
\end{align}
The key intuition is that $\text{sign}(\kappa)$ informs us whether the target falls between the two bounds, thus allowing us to optimize for our desired coverage level $\alpha$. This expression makes no assumptions about the shape of the target noise distribution. It does not require the interval bounds to be placed at specific quantile values, nor does it require the neural network to explicitly model all quantiles. As neither $\mu_1$ nor $\mu_2$ are explicitly tied to being the upper or lower bound, we select these bounds to be $\max(\mu_1, \mu_2)$ and $\min(\mu_1, \mu_2)$ respectively. This makes the expression \textit{permutation invariant} and thus avoids the crossing interval problem faced by quantile regression methods \cite{park2022learning} (see \Cref{sec:crossingbounds} for further discussion on this point).
\Cref{thm:RQRcoverage} provides formal guarantees that the minimization of this loss function in expectation yields a valid interval i.e. the coverage rate $\alpha$ is achieved. Note that proofs for all theorems and propositions are provided in \Cref{app:proofs}.
\begin{restatable}[RQR In-sample Coverage]{theorem}{RQRcoverage} For any random variable $Y$ associated with an input $x$, $\forall \, \alpha \in [0,1],$ 
\begin{multline*}
(\mu_1^*(x),\mu_2^*(x)) = \argmin_{\mu_1,\mu_2}{\{ \mathbb{E}_Y(\mathcal{L}^{\textup{RQR}}_{\alpha}((\mu_1,\mu_2), x, Y))\}}\IfRestatedTF{}{\\}\implies \mathbb{P}(\mu_1^*(x) < Y < \mu_2^*(x)) = \alpha. \\
\end{multline*}
\label{thm:RQRcoverage}
\end{restatable}
\vspace{-1.25cm}

Furthermore, in Theorem \ref{thm:RQRubiased} we derive additional desirable properties of the RQR objective. Specifically, that it obtains correct coverage in expectation with a variance bound that decreases in proportion to the inverse data size (i.e. increasing data size implies lower variance). This proof relies on first showing that this objective achieves the correct finite in-sample coverage -- which we provide in Theorem \ref{thm:finiteRQR} of \cref{app:proofs}.

\begin{customthm}{3.3}[RQR is Unbiased and Consistent]\label{thm:RQRubiased}
 For any random variable $Y$ associated with an input $x$, we consider N realizations of this random variable : $\{y_i\}_{i=1}^{N}$. $(\mu_{1,N},\mu_{2,N})$ are the bounds of our estimator trained on these N samples. $ \forall \, \alpha \in [0,1]$, we name $Q_N$ the absolute true miscoverage of our estimator  
 \begin{equation*}
     Q_N = |\int_{\mu_{1,N}}^{\mu_{2,N}}dPy - \alpha|
 \end{equation*}

Then $E[Q_N] = 0$ and the variance of $Q_N$ is bounded by $\frac{1}{4N}$.
\end{customthm}

These theoretical properties ensure that the RQR expression is well-motivated and provides a suitable objective for the goal of generating well-calibrated prediction intervals in practice. We further illustrate this by analyzing the objective's gradients in \cref{sec:analytic}.

\textbf{Rewarding preferable solutions.} Given the added flexibility of shifting intervals rather than bounding them around the median, there now exists a potentially infinite number of competing solutions that achieve the desired level of coverage. We note that the RQR expression is largely agnostic to solutions that obtain additional desirable properties such as narrower interval widths or improved conditional coverage which have been discussed extensively in previous works \citep[see e.g.][]{feldman2021improving, tagasovska2019single}. Therefore, we would like to induce some preference among solutions. A natural strategy is to upweight preferable solutions (e.g. narrower intervals or improved conditional coverage) via an additive regularization term $\mathcal{R}$ and a scalar weighting term $\lambda \in \mathbb{R}^+$. Thereby we provide practitioners with the flexibility to choose whichever interval properties provide the most utility. Therefore, the complete RQR-$\mathcal{R}$ (regularized) objective takes the form

\begin{equation} \label{eqn:RQR-R}
    \mathcal{L}^{\text{RQR-}\mathcal{R}}_\alpha((\mu_1,\mu_2), \mathbf{x}, y)\! =\! \mathcal{L}^{\text{RQR}}_{\alpha}((\mu_1,\mu_2), \mathbf{x}, y) + \lambda \cdot \mathcal{R}.
\end{equation}

Given this structure, we now introduce two specific choices for $\mathcal{R}$.

\textbf{\circled{1} RQR-W: Width Minimizing $\mathcal{R}$.} As discussed in \Cref{sec:background} \& \Cref{sec:intervalwidth}, interval width is a principal criterion for evaluating methods for obtaining predictive intervals \citep{tagasovska2019single, feldman2021improving, romano2019conformalized}. A direct approach for minimizing interval width is to penalize the sample-wise squared interval width such that 
\begin{multline*}
    \mathcal{L}^{\text{RQR-W}}_\alpha((\mu_1,\mu_2), \mathbf{x}, y) = \mathcal{L}^{\text{RQR}}_{\alpha}((\mu_1,\mu_2), \mathbf{x}, y)\\ + \lambda\frac{(\mu_2 - \mu_1)^2}{2}.
\end{multline*}
\begin{equation}
    P(Y \in (\mu_1, \mu_2)) = \alpha
\end{equation}
By integrating this penalty term, we aim to effectively navigate the landscape of potential solutions, encouraging the model to prioritize intervals of reduced length whilst still maintaining the targeted level of coverage. Analogous to the approach taken in \Cref{thm:RQRcoverage}, we can extend our analysis to consider this complete loss function. We find that the introduced penalty term induces a bias to the coverage of the interval estimator -- the result is now $\mathbb{P}(\mu_1^* < Y < \mu_2^*) = \alpha -2\lambda$. However, we can easily remove the bias by modifying the targeted coverage rate. By choosing $\hat{\alpha} = \alpha + 2\lambda$, we obtain $\mathbb{P}(\mu_1^* < Y < \mu_2^*) = \hat{\alpha} -2\lambda = \alpha$ as desired.

\begin{customthm}{3.4}[RQR-W In-sample Coverage]\label{thm:WMcoverage}
For any random  variable $Y$ associated with an input x, $\forall \, \alpha \in [0,1],$
\begin{multline*}
(\mu_1^*(x),\mu_2^*(x)) = \argmin_{\mu_1,\mu_2}{\{ \mathbb{E}_Y(\mathcal{L}^{\text{RQR-W}}_{\alpha+2\lambda} (\mu_1,\mu_2),x,Y))\}}\\\implies \mathbb{P}(\mu_1^*(x) < Y < \mu_2^*(x)) = \alpha 
\end{multline*}
\end{customthm}
\vspace{-0.5cm}

In the supplemental material we also show that, analogous to RQR, RQR-W obtains correct finite in-sample coverage (Theorem \ref{thm:finiteRQRW}) and is unbiased and consistent (Theorem \ref{thm:RQRWasymptoticdata}).
An important additional benefit of the penalty term, presented in Proposition \ref{prop:WMunique}, is that it makes the loss function convex for a wide range of target distributions and $\lambda$. Hence, leading to a welcome additional result: the existence and uniqueness of its minimum.

\begin{customprop}{3.7}[Existence and Uniqueness of Solution]\label{prop:WMunique}
$\mu_1^{min}$ and $\mu_2^{max}$ denote the bounds of our optimization problem. For a target distribution $Y$ with a cumulative distribution function that is k-Lipschitz continuous with $k < 1+\frac{\alpha}{\mu_2^{max}-\mu_1^{min}}$, when $\lambda > max(0,\int_{\mu_1^{min}}^{\mu_2^{max}} d\mathbb{P}_Y(y) - \alpha)$, the minimum of $\mathcal{L}^{\textup{RQR-W}}_{\alpha+2\lambda}$ exists and is unique.
\end{customprop}

 We later empirically verify (see \Cref{sec:experiments}) that this objective does perform well on real-world data. 

\textbf{\circled{2} RQR-O: Width-Coverage Independence $\mathcal{R}$.} As discussed in \Cref{sec:background}, whilst interval construction methods are most commonly evaluated based on their resulting interval width for a realized coverage level (sometimes referred to as the high-quality principle \citep{pearce2018high}), alternative objectives may also be desirable. \citet{feldman2021improving} introduced \textit{orthogonal quantile regression} (OQR) which instead optimized for a notion of \textit{conditional coverage} rather than minimizing interval width. Specifically, the authors introduce a regularization term that promotes independence between the size of the intervals and occurrences of a (mis)coverage event. They combine their proposed regularization term with QR and report significant gains in measures of conditional coverage. We can easily combine their term with our RQR instead by setting it as the regularization term $\mathcal{R}$ in \Cref{eqn:RQR-R} where
\begin{equation*}
    \mathcal{R}(\cdot) = \left|\frac{\text{Cov}(\mathbf{w}, \mathbf{m})}{\text{Var}(\mathbf{w})\text{Var}(\mathbf{m})}\right|.
\end{equation*}
With $\mathbf{w}$ denoting the vector of interval widths where ${w}_i = |\mu_2(\mathbf{x}_i) - \mu_1(\mathbf{x}_i)|$ and $\mathbf{m}$ denoting the indicator vector of coverage events where ${m}_i = \mathbb{I}_{y_i \in [\mu_1(\mathbf{x}_i), \mu_2(\mathbf{x}_i)]}$ -- both calculated on the training data. Thus, this regularization term can simply be interpreted as the Pearson correlation between the interval widths and instances of coverage or miscoverage\footnote{Note that taking the absolute value results in this term penalizing correlation between interval width and instances of \textit{either} coverage or miscoverage (as we would desire).}. We refer to this complete objective as RQR-O (orthogonal) and, as with the other objectives, we provide proof for its expected coverage in \Cref{app:proofs}. 

\textbf{Trading-off width and orthogonality.} We note that there is typically a trade-off between minimizing interval width and maximizing conditional coverage. As we later observe in \Cref{fig:kin8nm}, obtaining near optimal conditional coverage generally requires wider intervals than is strictly necessary for obtaining valid marginal coverage. We emphasize that in this work we are agnostic as to which qualities are preferable and, instead, we enable the practitioner to make this decision based on their specific application. In \cref{sec:experiments}, we will empirically verify that our proposed objective behaves as expected and successfully utilizes its added flexibility to outperform benchmark methods at achieving their respective goals whilst maintaining empirical coverage.

\subsection{Gradient Analysis of the RQR objective}\label{sec:analytic}
We now further analyze how our proposed objective in \cref{eqn:RQRbaseobjective} behaves analytically by investigating a concrete setting. In particular, let us consider a specific interval for a given observation $x$ and suppose that the bounds $(\mu_1, \mu_2) = (0, 1)$ are output by a model. We are now interested in how these two bounds will be updated by gradient descent if we observe various potential values of the ground truth label $y$. This is captured by the gradients $\frac{\partial \mathcal{L}}{\partial \mu_1}$ and $\frac{\partial \mathcal{L}}{\partial \mu_2}$ which describe the \textit{direction} and \textit{relative magnitude} in which each bound will move.

\begin{figure}
    \centering
    \includegraphics[width=0.95\columnwidth]{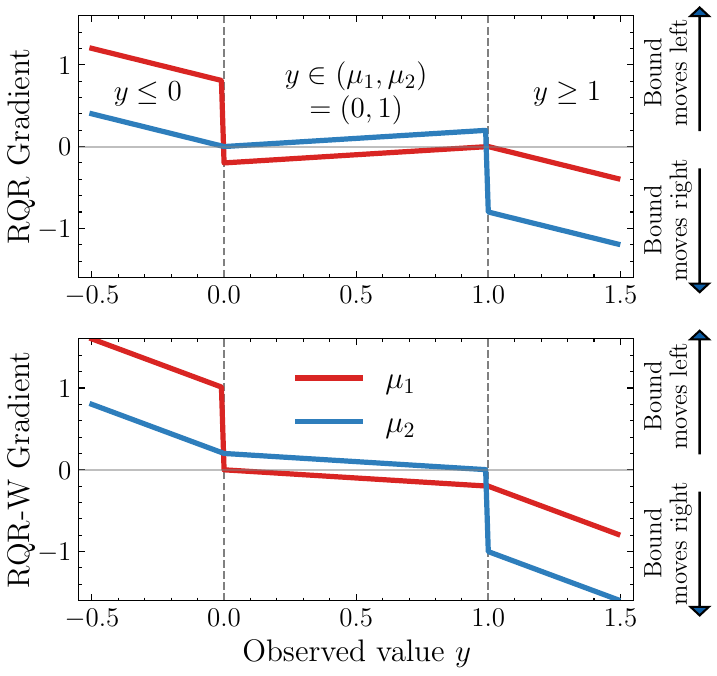}
    \caption{\textbf{Gradient analysis.} The gradients with respect to the RQR objective (\textit{upper}) and RQR-W objective (\textit{lower}) for a single $x$ over a range of potential values of $y$. The current predicted bounds $(\mu_1, \mu_2) = (0, 1)$ are updated using these gradients.}
    \label{fig:gradients}
\end{figure}
In \cref{fig:gradients} (upper) we provide a plot capturing the analytic value of these respective gradients for RQR at a range of potential values of the ground truth $y$. Positive gradient values (i.e. the upper half of the subplot) indicate that a given bound shifts to the left on the number line while, conversely, negative gradient values (i.e. the lower half of the subplot) indicate that a given bound shifts to the right. Then the absolute values capture the magnitude of how much that bound will shift. 

In the left-hand side region ($y \leq 0$) where the $y$ values fall outside on the left of the interval, we note that both bounds will shift to the left, with the nearer (lower) bound moving by the greatest magnitude. Similarly, in the right-hand side region ($y \geq 1$) we observe the opposite where both bounds move to the right. Finally, when the target does fall inside the interval ($y \in (0, 1)$) the bounds are adjusted by a smaller magnitude and are narrowed. When applied collectively across the entire training set, this results in the intervals being widened for miscoverage events and narrowed for coverage events as we might desire.

In contrast, in \cref{fig:gradients} (lower), when we include the width-reducing regularization term (i.e. the RQR-W objective with $\lambda = 0.1$) this plot changes. The key difference in this case is in the center region where the target falls inside the interval ($y \in (0, 1)$). As with the upper figure, the interval still narrows – but now the further bound from the target $y$ has the larger gradient allowing the interval to narrow by a greater amount before a miscoverage event will occur. In other words, while previously the nearest bound to the target $y$ narrowed the most, now the further bound narrows by a greater magnitude.

Given that this analysis shows a narrowing for coverage events and a widening for miscoverage events \textit{for a single example}, we might ask how the solution of the respective objectives differ \textit{in aggregate} across the entire training data in which we observe both cases. In \cref{app:specialcase} we provide a special case where we can derive a closed-form solution in which we find that the RQR-W objective does indeed obtain a narrower solution than RQR with equal coverage.

\section{Experiments} \label{sec:experiments}
\begin{table}
\caption{\small \textbf{Empirical verification}. Coverage and width ($\pm$ standard errors) are assessed for producing predictive intervals on symmetric and non-symmetric noise distributions with $\alpha = 0.8$. As expected, for realistic, skewed noise distributions, width-minimizing RQR (RQR-W) produces narrower intervals.}
\label{tab:synthetic}
\centering
\resizebox{\columnwidth}{!}{%
\begin{tabular}{c|ccc} \toprule
     & Coverage & Width & Dist. Histogram \\ \midrule
     QR & 0.81 ($\pm$0.007) & 0.26 ($\pm$0.003) & \multirow{3}{*}{\parbox[b]{2cm}{\includegraphics[width=\linewidth, height=1cm]{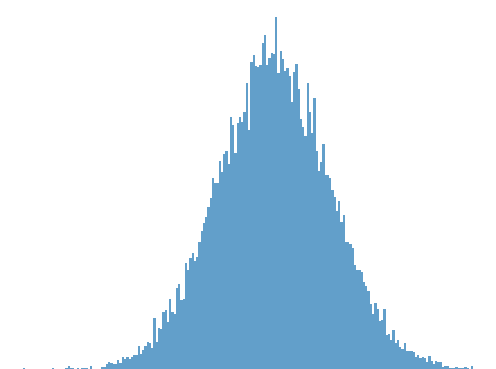}}} \\
     RQR (w/o reg) & 0.81 ($\pm$0.003) & 0.26 ($\pm$0.002) & \\
     RQR-W (with reg) & 0.82 ($\pm$0.004) & 0.26 ($\pm$0.002) & \\ \midrule
     QR & 0.80 ($\pm$0.002) & 1.61 ($\pm$0.007) & \multirow{3}{*}{\parbox[b]{2cm}{\includegraphics[width=\linewidth, height=1cm]{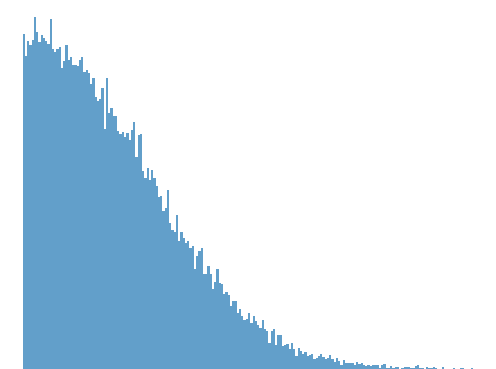}}} \\
     RQR (w/o reg) & 0.80 ($\pm$0.007) & 1.72 ($\pm$0.010) & \\
     RQR-W (with reg) & 0.81 ($\pm$0.002) & 1.50 ($\pm$0.005) & \\ \bottomrule
\end{tabular}}
\vspace{-0.25cm}
\end{table}
\textbf{Empirical verification.}\footnote{Code at \url{https://github.com/TPouplin/RQR}.} We begin by empirically verifying the predicted gains of the RQR objective over quantile regression on non-symmetric noise distributions. As the noise distribution cannot be known on real-world data, here we generate synthetic data according to a known process. The data is generated according to a data-generating process in which the label is determined according to $Y = X + \epsilon$, where $X \in \mathbb{R}$ represents a deterministic component which is set as constant and $\epsilon \in \mathbb{R}$ is the noise component. Then the noise distribution is selected as either a (symmetric) Gaussian or a (non-symmetric) truncated Gaussian. We fit a simple linear neural network consisting of just a single layer. As illustrated in \Cref{tab:synthetic}, the empirical results match our theoretical expectations. All methods perform equally well on the symmetric Gaussian noise where intervals centered at the median are optimal. However, QR fails to achieve optimal width on the truncated Gaussian due to being arbitrarily centered at the median. Whilst RQR (w/o reg) is unbiased, it is not sufficiently incentivized to produce narrower intervals and thus performs similarly to QR in that regard. In line with our analytic observations that motivated this objective, only RQR-W (with reg) achieves the optimally narrow interval solution.

\begin{figure}[h]
    \centering

    \begin{subfigure}{\columnwidth}
    \includegraphics[width=\columnwidth]{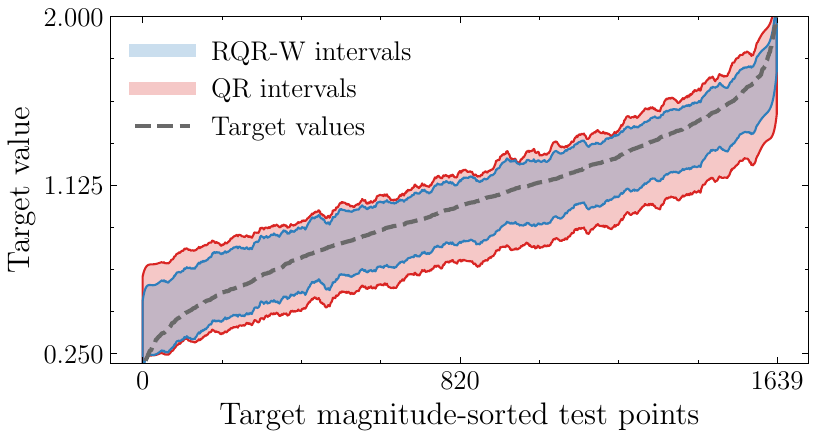}
    \end{subfigure}

    \begin{subfigure}{\columnwidth}
    \includegraphics[width=\columnwidth]{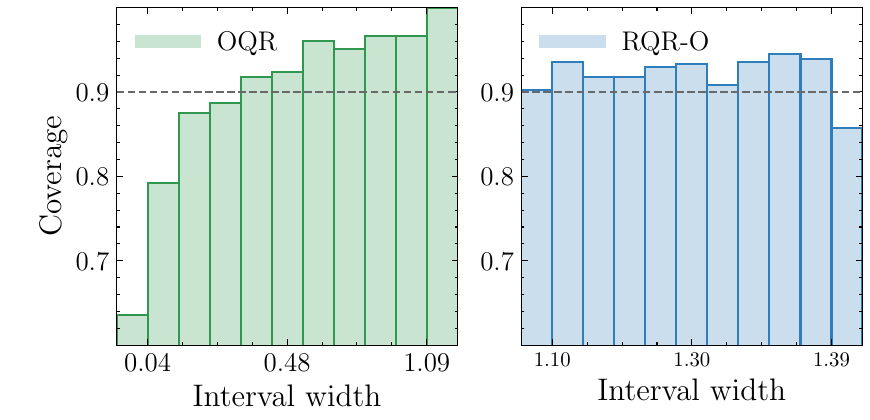}
    \end{subfigure}

    \caption{\small \textbf{Applying RQR-W \& RQR-O in practice.} Resulting intervals on robotics distance estimation task. \textbf{Upper:} The RQR-W objective achieves generally narrower intervals across the 1639 test examples. \textbf{Lower:} The RQR-O objective achieves more consistent coverage across different interval widths.}
    \label{fig:kin8nm}
\end{figure}

\textbf{Applying RQR-W \& RQR-O in practice.} We now turn our attention to applying our proposed interval prediction approach to practical tasks in high-stakes domains. To illustrate, we consider a task in the robotics domain in which safety and reliability are essential, especially when applied in close proximity to humans \citep{baek2023safety}. We consider the task of estimating a robot arm-effector's distance from a target given noisy measurements of inputs such as joint positions or twist angles. The \textit{kin8nm} dataset provides an example of such a task consisting of 8192 measurements \citep{ghahramani1996kin}. For this problem, we desire accurate intervals such that the arm may be used safely and effectively and, depending on the specifics of the application, either interval width or conditional coverage may be important. In \Cref{fig:kin8nm} we provide the results of applying RQR to this problem using the same experimental setup as described later in the benchmarking experiments. 

In the \textit{upper} subplot, we optimize for interval width using RQR-W and compare the resulting intervals to those using QR. For presentation we sort the test set points according to increasing target magnitude and apply a Savitzky–Golay filter to smooth the intervals (see \Cref{fig:unsmoothkin8nm} for a version without smoothing applied). We observe that the RQR-W intervals are noticeably narrower despite providing the same marginal coverage. Then, in the \textit{lower} subplot, we demonstrate the effects of instead optimizing for conditional coverage using RQR-O and evaluate the empirical coverage when the test set is divided into subgroups based on interval width. In the left of the two histograms, we observe that the baseline method, OQR, under-covers for narrower intervals whilst over-covering for wider intervals. This is in contrast to RQR-O on the right which achieves generally balanced coverage across all interval widths. Given that conditional coverage is the exclusive goal in this case, the RQR-O solution discovers that wider intervals are necessary to achieve this. As a result, it is able to ensure that the probability of error is not dependent on interval width resulting in a more consistent coverage across the data.

Therefore, unlike previous works, RQR may be viewed as a general-purpose approach to constructing intervals in which the practitioner may choose to prioritize among additional interval properties depending on a particular application. In this particular robotics example, minimizing interval width may result in a more accurate estimate of the object's distance while improving conditional coverage may help prevent specific failure modes.

\begin{table}
\vspace{-0.25cm}
\caption{\small \textbf{Benchmarking RQR-W}. We evaluate each baseline method against RQR-W across 12 datasets with the aggregated results presented. RQR-W obtains the desired coverage level more frequently (coverage obtained) with a lower magnitude of miscoverage from the desired level (mean miscoverage) and achieves the lowest or joint lowest ($\pm$ 1 standard error) mean interval width more frequently (narrowest intervals) across the benchmark tasks.}
\label{tab:widthresults}
\centering
\resizebox{\columnwidth}{!}{
\begin{tabular}{cccccc} \toprule
     & RQR-W & QR & SQR-C & SQR-N & IR \\ \midrule
    Coverage obtained &  \textbf{12/12} &  \textbf{12/12} &  9/12 &  5/12 &   11/12 \\
    Mean miscoverage (\%) &  \textbf{0.93} &  1.45 &  2.50 &  3.73 &   2.88 \\
    Narrowest intervals &  \textbf{9/12} &  8/12 &  0/12 &  1/12 &   0/12 \\
    \bottomrule
\end{tabular}}
\vspace{-0.25cm}
\end{table}

\textbf{Benchmarking RQR-W.} We now proceed to investigate the performance of RQR-W more broadly on the standard quantile regression benchmark tasks used in \citet{tagasovska2019single, chung2021beyond, pearce2018high} consisting of nine datasets from \citet{asuncion2007uci}. We extend this benchmark, as suggested by the conference reviewers, with three additional datasets from \citet{grinsztajn2022tree}. We follow the preprocessing and experimental protocol described in \Cref{app:experimentdetails} in line with previous works. To summarize, we train two-layer neural networks using a grid search to find optimal hyperparameters. All experiments are repeated over 10 seeds for each dataset. For building SQR-C intervals, we followed the method prescribed by the authors which consists of selecting the symmetric (0.05, 0.95) intervals. We provide a summary of our results in \Cref{tab:widthresults}. We report the proportion of datasets in which each method obtains coverage (within a 2.5\% margin of error). We then report the mean absolute miscoverage distance from the desired level (90\%) across the 12 datasets. Finally, we also provide the frequency with which each method achieves the (possibly joint) narrowest mean interval width. These results illustrate that, on aggregate, the added flexibility of RQR-W is better able to balance our desire for narrow intervals \emph{whilst also} obtaining our desired level of coverage than existing baselines. Gains in realized coverage are a convenient side effect of two factors: (1) quantile regression introduces two sources of error (i.e. estimating each of the two quantiles) in which an error on \emph{either} estimate can result in miscoverage. This is not the case for methods that estimate the interval directly\footnote{Note that we are not the first to notice this fact, see e.g. Sec. 2 of \citet{takeuchi2006nonparametric} for a much earlier work that has commented on this limitation.}; (2) minimizing interval width selects intervals that are likely to lie in denser regions and are likely to exhibit lower variance in their estimates (see \Cref{sec:intervalwidth} for an extended discussion on this point).

The complete results for each dataset are included in \Cref{sec:RQRWcomplete} where we provide the resulting coverage and MPIW with means and standard errors of means reported. There we also include histograms of the target distributions for each dataset to highlight the point that non-symmetric distributions should be the \textit{standard expectation} on real-world regression tasks. Note that the target coverage level is set to 90\% throughout our experiments and we ablate for other coverage levels in \Cref{sec:otheralpha} where we observe similar results. Overall, the results in this section demonstrate that the added flexibility of circumventing the standard step of learning predefined quantiles can be utilized to obtain narrower intervals in practice.

\textbf{Benchmarking RQR-O.} In a similar vein, we evaluate RQR-O for its effectiveness in improving measures of conditional coverage. In this case, we compare to the original OQR work of \citet{feldman2021improving}. We evaluate both methods on the same benchmark datasets as previously and follow the same experimental protocol as in \citet{feldman2021improving}. The key distinction in this experiment is that the coefficient of the regularization parameter for both methods is incrementally reduced until empirical coverage is achieved. This is because a comparison of conditional coverage using these metrics is only meaningful if both methods achieve a similar level of empirical marginal coverage. Again, the experimental setup is described in detail in \Cref{app:experimentdetails}. We investigate this regularization term's effectiveness at achieving its stated objective of enforcing orthogonality between interval width and instances of miscoverage. Since both methods use an identical regularization expression, the gains obtained by RQR-O in this section are due to pairing this term with our RQR objective from \Cref{eqn:RQRbaseobjective} rather than existing quantile regression objectives that suffer from the limitations discussed in \Cref{sec:limitations}. The results are provided in \Cref{tab:orthogonalresults} where we compare performance based on a test set evaluation of coverage, $\%$ improvement in Pearson correlation over OQR, and $\%$ improvement in HSIC over OQR. We generally find that RQR-O achieves a significant improvement in these measures of conditional coverage, again indicating that the added flexibility of this approach enables a more favorable solution to be found. An important takeaway from these results is that, since Pearson's correlation is the regularization objective used by both methods, the gains in performance when considering it as an evaluation metric provide direct evidence that the added flexibility provided by the RQR loss function (due to not being centered around the median) enables it to find a better solution in the auxiliary task (in this case, minimizing Pearson's correlation between instances of miscoverage and interval width).

\begin{table}
\caption{\small \textbf{Benchmarking RQR-O.} We compare our proposed loss function to those used in OQR for achieving improved conditional coverage as evaluated using standard metrics (see \Cref{sec:background}). We report the empirical coverage achieved and the \% improvement obtained over OQR in Pearson correlation and HSIC ($\pm$ a standard error).}
\label{tab:orthogonalresults}
\centering
\resizebox{\columnwidth}{!}{%
\begin{tabular}{c@{\hskip 5pt}c@{\hskip 5pt}c@{\hskip 5pt}c@{\hskip 5pt}c} \toprule
    \multirow{2}{*}{Dataset} & \multirow{2}{*}{\shortstack[c]{RQR-O (ours)\\coverage}} & \multirow{2}{*}{\shortstack[c]{OQR\\coverage}} & \multirow{2}{*}{\shortstack[c]{Pearson\\correlation}} & \multirow{2}{*}{HSIC} \\ 
    &&&& \\ \midrule
    {concrete} &   88.75 (0.14) &  87.54 (0.12) &    +80.19 (0.72) &    +98.01 (0.12)  \\
    {power} &   90.00 (0.05) &  91.86 (0.05) &    +86.49 (0.63) &    +77.79 (1.19) \\
    {wine} &  89.15 (0.27) &  88.59 (0.10) &    +71.42 (1.14) &    +96.77 (0.29)  \\
    {yacht} &  88.71 (0.99) &  89.07 (0.20) &  -45.80 (11.28) &  -18.49 (10.7) \\
    {naval} & 90.58 (0.07) &  90.12 (0.07) &    +87.83 (0.48) &    +16.87 (2.95) \\
    {energy} & 89.77 (0.18) &  90.42 (0.09) &    +25.04 (2.70) &    +76.57 (2.12) \\
    {boston} & 90.49 (0.18) &  91.97 (0.15) &    +75.97 (0.80) &    +85.67 (1.54)  \\
    {kin8nm} & 90.40 (0.07) &  89.94 (0.05) &    +89.16 (0.30) &    +99.91 (0.01) \\ 
    {protein} &89.87 (0.02) &  89.45 (0.03) &    +80.29 (0.58) &    +99.91 (0.01)  \\ \bottomrule
\end{tabular}}
\vspace{-0.25cm}
\end{table}

\section{Conclusion}
In this work we have introduced \textit{Relaxed Quantile Regression}, a direct alternative to quantile regression that circumvents the requirement to prespecify the exact quantiles being learned whilst maintaining its attractive coverage properties. We then demonstrated that this new loss can be easily combined with user-specified regularization terms to obtain a solution suitable for a given application (i.e. narrower intervals or improved conditional coverage). Finally, we evaluated the method against state-of-the-art single-model methods across standard benchmark tasks demonstrating that this added flexibility is converted into improved performance in terms of either narrower intervals or better conditional coverage, depending on the practitioner's preference. 

For future work, devising new approaches to evaluating predictive intervals and developing complimentary regularization terms such that they can be paired with RQR provides an opportunity to extend upon this work. Additionally, applying a post hoc conformalization procedure to an underlying quantile regression predictor has already been demonstrated to be highly effective (e.g. \mbox{\citet{romano2019conformalized,sesia2021conformal}}). Although the choice of base estimator is acknowledged to impact the performance of the overall procedure \mbox{\cite{sesia2021conformal}}, to the best of our knowledge, no work has systematically investigated the interaction between the two\footnote{One work that is relevant to this direction is \mbox{\citet{sesia2020comparison}} which does consider different conformalization procedures with a fixed quantile regression method but not vice versa.} -- thus providing a valuable direction for future research.

\clearpage

\subsection*{Impact Statement}
This paper presents work whose goal is to advance the field of Machine Learning. There are many potential societal consequences of our work, none of which we feel must be specifically highlighted here.

\subsection*{Acknowledgments}
We would like to thank Fergus Imrie, Julianna Piskorz, Johannes Vallikivi, Alicia Curth, and the anonymous reviewers for very insightful comments and discussions on earlier drafts of this paper. TP would like to thank AstraZeneca for their sponsorship and support. AJ and NS gratefully acknowledge funding from the Cystic Fibrosis Trust. This work was supported by Azure sponsorship credits granted by Microsoft’s AI for Good Research Lab.

\bibliography{bibliography}
\bibliographystyle{icml2024}

\newpage
\appendix
\onecolumn

\section{Extended Related Work} \label{app:extendedrelatedwork}
In this work, we have focused on \textit{single model approaches} (i.e. quantile/interval regression) for obtaining predictive intervals from neural networks with the intention of developing more effective methods within this category. Due to the vital importance of uncertainty quantification, a vast literature of disparate alternative approaches has emerged, each representing interesting research directions with their own respective strengths and weaknesses. In this section, we provide a broad overview of the leading methods for obtaining predictive intervals from neural networks. Given the extensive nature of this topic, this summary is not exhaustive and we refer the reader to the more comprehensive references cited within each topic for a more complete overview. For a recent survey on predictive intervals in regression problems more generally (i.e. beyond neural networks) we refer the reader to \citet{dewolf2023valid} or for general neural network uncertainty quantification to \citet{gawlikowski2023survey}.

\textbf{Single model approaches.} The subject of this work is individual models that output intervals rather than point estimates which we refer to as single model approaches. An in-depth description of the state-of-the-art methods within this category is provided in \Cref{sec:background}. Whilst these methods typically estimate aleatoric uncertainty, it is quite straightforward to combine them with Bayesian or ensemble methods to simultaneously account for epistemic uncertainty \citep{tagasovska2019single}. As previously discussed, quantile regression \citep{koenker1978regression} has excelled in this category -- including prior to and outside of the deep learning regime \citep{koenker2001quantile, meinshausen2006quantile, yu1998local}. Similarly, the evaluated direct interval estimation method of \citet{pearce2018high} was also an evolution of the foundational work of \citet{khosravi2010lower}. \citet{chung2021beyond} primarily focuses on the task of training models that output the full predictive distribution (rather than intervals for a preselected coverage level) which suffers from the same limitations as SQR that were discussed in the main text. Elsewhere in the time series setting, \citet{gasthaus2019probabilistic} estimate quantiles using monotonic regression splines. Whilst strictly speaking a dual model approach, variance networks fit a second neural network to estimate the variance of a prediction \citep{skafte2019reliable}. A very simple approach that can be easily combined with most other approaches is to regularize the neural network such that measures of calibration are optimized using e.g. confident output penalization \citep{pereyra2017regularizing}, label smoothing \citep{szegedy2016rethinking}, or induced label noise \citep{xie2016disturblabel}.

\textbf{Bayesian methods.} Bayesian neural networks attempt to model the target probability distribution for a given test example $x$ given some observed data $\mathcal{D}$ by marginalizing over a distribution of network parameters $\theta$ such that $P(y|x,\mathcal{D}) = \int P(y|x, \theta) P(\theta|\mathcal{D})$ from which predictive intervals can be derived. This expression requires intractable calculations which may be approximated in practice using various techniques. A simple approach consists of taking a second-order Taylor expansion around the maximum a posteriori estimate of $\theta$ to produce a Gaussian approximation of $P(\theta| \mathcal{D})$ known as the \textit{Laplace approximation} \citep{tierney1986accurate}. In practice, further scaling efforts have been required to apply this method in the modern deep learning context \citep{daxberger2021laplace}. Alternatively, \textit{variational inference} substitutes $P(\theta| \mathcal{D}) \approx q(\theta)$ where $q(\theta)$ denotes some tractable parametric approximation such as a Gaussian distribution \citep{hinton1993keeping}. Significant research has investigated methods for extending this approach to account for modern datasets and architectures \citep{graves2011practical, zhang2018advances}. \textit{Monte Carlo integration} instead approximates the integral over parameters with a finite sum such that $\int P(y|x, \theta) P(\theta|\mathcal{D}) \approx \frac{1}{M}\sum_{j=1}^{M} P(y|x, \theta_j)$ \citep{caflisch1998monte}. Then different choices of selecting a subset of $M$ weight parameterizations result in alternative instantiations of this approximation (see e.g. Ch. 17 of \citet{goodfellow2016deep}). \textit{Monte Carlo dropout}, which drops neurons at test time according to a Bernoulli distribution to estimate uncertainty, has become popular due to its conceptual and implementation simplicity \citep{gal2016dropout}. A somewhat distinct Bayesian approach is the \textit{Gaussian process} which is a collection of random variables of which any finite sample is Gaussian distributed specified by a specific mean function and covariance kernel \citep{rasmussen2006gaussian}. Although this method suffers from important limitations (e.g. scaling), its adaption to the deep learning setting has achieved notable performance on benchmark tasks \citep{wilson2016deep, wilson2016stochastic}.

\textbf{Deep ensembles.} Aggregating outputs over a set of neural networks has emerged as a simple but effective method that accounts for epistemic uncertainty \citep{lakshminarayanan2017simple}. Whilst this approach can be considered studied through a Bayesian perspective \citep{wilson2020bayesian, wilson2020case}, it has primarily developed from the classical ensembling literature \citep{sagi2018ensemble}. It is hypothesised that the empirical success of deep ensembles is due to their better exploration of the loss landscape \citep{fort2019deep} -- with diversity typically achieved through random initialization and batching due to known challenges in optimizing for exploration \citep{jeffares2023joint, abe2023pathologies}. However, recent works have suggested that their improved calibration may be overstated \citep{rahaman2021uncertainty} and that performance increases may be better understood as being due to an increased model capacity \citep{abe2022deep}. These gains also come at the cost of a significant computational overhead with the relative computational cost growing linearly with ensemble size. More efficient approaches to deep ensembling have been proposed in recent years to reduce this cost \citep[e.g.][]{wen2019batchensemble}.


\textbf{Parametric approaches.} The classical statistics literature has a long history of developing principled estimates of predictive intervals derived from parametric assumptions in the data generating process \citep[see e.g. Ch. 5.3][]{seber2003linear}. One such approach in the neural network setting is the delta method which makes a linearity assumption in the region around a prediction paired with a Gaussian assumption on the noise distribution \citep{hwang1997prediction, khosravi2011comprehensive}. Another example is \citet{nix1994estimating} who also assume the noise distribution to be Gaussian and derive a cost function to estimate its value with an auxiliary output to the network. 

\textbf{Post-hoc methods.} Several methods exist in which calibrated intervals are constructed or updated as a post-processing step for an existing point predictor or interval estimator respectively. Perhaps the most notable of these is conformal prediction \citep{vovk2005algorithmic} and, in particular, inductive conformal prediction \citep{papadopoulos2008inductive}, which provides prediction intervals with finite sample marginal coverage guarantees. \citet{romano2019conformalized} further developed an approach that also performs well conditionally (i.e. intervals where width is adaptive to a given example). As noted by the authors, this method ``can wrap around any algorithm for quantile regression'', thus making it a complimentary post-hoc approach. Similarly, \mbox{\citet{sesia2021conformal}} use (multiple) quantile regression as a base estimator for conformalization which they report to ``work particularly well''. Other methods typically focus on improving the calibration of an underlying point predictor. Approaches include Platt scaling \citep{platt1999probabilistic}, temperature scaling \citep{tomani2021post}, histogram binning \citep{zadrozny2001obtaining}, test time augmentation \citep{hekler2023test}, and isotonic regression \citep{zadrozny2002transforming}.

\begin{table}
\caption{\textbf{Categories of interval construction approaches.} A broad comparison of different categories of interval construction highlighting the key distinctions of single model approaches. Detailed descriptions and discussion is provided in the text.}
\label{tab:dantablecategories}
\centering
\begin{tabular}{c|ccccc} \toprule
     & \textbf{(1)} & \textbf{(2)} & \textbf{(3)} & \textbf{(4)}\\ \midrule
    \multirow{2}{*}{Single model approaches}  & \multirow{2}{*}{\cmark} & \multirow{2}{*}{\cmark} & \multirow{2}{*}{\cmark} & \multirow{2}{*}{\cmark}  \\
    &&&& \\
    \multirow{2}{*}{Bayesian methods}  & \multirow{2}{*}{\xmark}  & \multirow{2}{*}{\xmark} & \multirow{2}{*}{\cmark}  &  \multirow{2}{*}{\cxmark} \\
    &&&& \\
    \multirow{2}{*}{Deep ensembles}  & \multirow{2}{*}{\xmark}  & \multirow{2}{*}{\xmark} & \multirow{2}{*}{\xmark} & \multirow{2}{*}{\cmark} \\
    &&&& \\
    \multirow{2}{*}{Parametric approaches} & \multirow{2}{*}{\cxmark}  & \multirow{2}{*}{\xmark} & \multirow{2}{*}{\cmark}  &  \multirow{2}{*}{\cxmark}   \\
    &&&& \\
    \multirow{2}{*}{Post-hoc methods} & \multirow{2}{*}{\cmark}  & \multirow{2}{*}{\cmark} & \multirow{2}{*}{\xmark}  & \multirow{2}{*}{\cmark}  \\
    &&&& \\ \bottomrule
\end{tabular}
\end{table}

\textbf{Comparing categories of interval construction.} We now provide some high level distinctions between these categories of approaches for constructing intervals. This is \textit{not} intended as a complete evaluation of competing uncertainty quantification approaches, rather we wish to highlight the advantages of single model approaches to emphasize the significance of developments \textit{within} this category. Furthermore, due to the broadness of these categories and the lack of clear boundaries between them, the following distinctions act as generalizations for which some exceptions exist. A discussion of these distinctions is provided in the next paragraph with a summary provided in \Cref{tab:dantablecategories}.

\textbf{(1)} \textit{Minimal computational overhead} - Deep ensembles require training a neural network from scratch $M$ times while Bayesian methods typically require approximations to produce tractable algorithms for large-scale models \citep{abdullah2022review, osawa2019practical}. Parametric methods require some overhead with the specific amount method dependent. In contrast, the other categories, including single model approaches, generally only require at most a change of loss function at training time. \textbf{(2)} \textit{Data-assumption-free valid intervals} - Parametric approaches and Bayesian methods generally require assumptions on the data-generating process to provide validity guarantees on their intervals. Deep ensembles don't provide such guarantees on derived intervals. However, many single model and post-hoc methods provide asymptotic or even finite sample guarantees of valid intervals. \textbf{(3)} \textit{Distinguishes between aleatoric \& epistemic uncertainty} - Explicitly differentiating between aleatoric (irreducible) and epistemic (reducible) uncertainty can be valuable \citep{hullermeier2021aleatoric}. As discussed in \cite{tagasovska2019single, pearce2018high}, the loss function of single model approaches generally estimates aleatoric uncertainty while epistemic uncertainty can also be accounted for by applying e.g. orthonormal certificates or interval ensembling. Explicitly modeling uncertainties in this way also tends to be at the heart of Bayesian and parametric methods (see e.g. \citet{kendall2017uncertainties}). Deep ensembles and post-hoc methods do not typically make this distinction. \textbf{(4)} \textit{Directly applicable to any loss-based algorithm} - Single-model approaches are typically characterized by simply replacing a point estimate loss function with a quantile regression loss (e.g. mean squared error $\to$ pinball loss). Similarly, deep ensembles only require running an algorithm multiple times and post-hoc methods typically wrap around or recallibrate arbitrary models. Some Bayesian and parametric methods are generally applicable (e.g. Monte Carlo dropout) however others require more substantial changes resulting in different algorithms when applied to neural networks (e.g. a Gaussian process).









\section{Theory and Proofs} \label{app:proofs}

\textbf{Equivalence of the Winkler Score and the Pinball Loss.}
\begin{proof}
For a targeted coverage level $1-\alpha$ and the bounds ($\mu_1$,$\mu_2$), we demonstrate that the Winkler score is proportional to the Pinball Loss by a factor $\frac{2}{\alpha}$.
\begin{align*}
    \frac{2}{\alpha}\mathcal{L}^{\text{QR}}_\alpha((\mu_1, \mu_2, \mathbf{x}, y) &= \frac{2\rho_{\frac{\alpha}{2}}(\mu_1(x), \mathbf{x}, y) + 2\rho_{1-\frac{\alpha}{2}}(\mu_2(x), \mathbf{x}, y)}{\alpha} \\
    & = \frac{2(\mathbb{I}_{y\leq \mu_1(x)} - \frac{\alpha}{2})(\mu_1(x) - y) + 2(\mathbb{I}_{y\leq \mu_2(x)} - 1 +\frac{\alpha}{2})(\mu_2(x) - y)}{\alpha} \\
    & = \frac{2(\mathbb{I}_{y\leq \mu_1(x)} - \frac{\alpha}{2})(\mu_1(x) - y) + 2(1 - \mathbb{I}_{y\leq \mu_2(x)}  -\frac{\alpha}{2})(y - \mu_2(x) )}{\alpha} \\
    & = \frac{2(\mathbb{I}_{y\leq \mu_1(x)} - \frac{\alpha}{2})(\mu_1(x) - y) + 2(\mathbb{I}_{y\geq \mu_2(x)}  -\frac{\alpha}{2})(y - \mu_2(x) )}{\alpha} \\
    & =(\mu_2(x) - \mu_1(x)) + \frac{2}{\alpha}\mathbb{I}_{y\leq \mu_1(x)}(q_{\frac{\alpha}{2}} - y)  + \frac{2}{{\alpha}}\mathbb{I}_{y\geq \mu_2(x)}(y - \mu_2(x) ) \\
\end{align*}
We obtain the expression of the Winkler Score \cite{winkler1972decision}.
\end{proof}

\RQRcoverage*
\begin{proof}
In these proofs, we omit explicitly including the input $x$ for clarity. However, the reader should recall that $Y$ is the random variable associated with the input $x$. 

First, we can rewrite our new loss with an indicator function: 
\begin{equation*}
    \mathcal{L}^{\text{RQR}}_{\alpha}((\mu_1,\mu_2), y)) = (y-\mu_1)(y-\mu_2)(\alpha-\mathbb{I}_{y \in [\mu_1,\mu_2]})
\end{equation*}

Then, we consider the expectation of the loss:
\begin{equation*}
\mathbb{E}( \mathcal{L}^{\text{RQR}}_{\alpha}((\mu_1,\mu_2), Y))) = \alpha \int_{-\infty}^{\infty} (y - \mu_1)(y - \mu_2)d\mathbb{P}_Y(y) - \int_{\mu_1}^{\mu_2} (y - \mu_1)(y - \mu_2)d\mathbb{P}_Y(y)
\end{equation*}

We find the expression of its partial derivatives with respect to the bounds:

\begin{equation*}
\begin{split}
\frac{\partial \mathbb{E}(\mathcal{L}^{\text{RQR}}_{\alpha})}{\partial \mu_1} & =  -\alpha \int_{-\infty}^{\infty} (y - \mu_2)d\mathbb{P}_Y(y) + \int_{\mu_1}^{\mu_2} (y - \mu_2)d\mathbb{P}_Y(y) \\
\frac{\partial \mathbb{E}(\mathcal{L}^{\text{RQR}}_{\alpha})}{\partial \mu_2} & =-\alpha \int_{-\infty}^{\infty} (y - \mu_1)d\mathbb{P}_Y(y) + \int_{\mu_1}^{\mu_2} (y - \mu_1)d\mathbb{P}_Y(y)
\end{split}
\end{equation*}

At the minimum, the gradient of the expected loss is null. Thus, if there is a minimum at the point $(\mu_1^*,\mu_2^*)$ with $\mu_2^* > \mu_1^*$,
\begin{equation*}
\begin{split}
& \frac{\partial \mathbb{E}(\mathcal{L}^{\text{RQR}}_{\alpha})}{\partial \mu_1}\Bigr|_{\mu_1^*,\mu_2^*} - \frac{\partial \mathbb{E}(\mathcal{L}^{\text{RQR}}_{\alpha})}{\partial \mu_2}\Bigr|_{\mu_1^*,\mu_2^*} = 0 \\
\implies & -\alpha \int_{-\infty}^{\infty} (\mu_1^* - \mu_2^*)d\mathbb{P}_Y(y) + \int_{\mu_1^*}^{\mu_2^*} (\mu_1^*  - \mu_2^*)d\mathbb{P}_Y(y) = 0\\
\implies & -\alpha (\mu_1^* - \mu_2^*)\int_{-\infty}^{\infty}d\mathbb{P}_Y(y) + (\mu_1^*  - \mu_2^*) \int_{\mu_1^*}^{\mu_2^*} d\mathbb{P}_Y(y) = 0 \\
\implies & -\alpha \int_{-\infty}^{\infty}d\mathbb{P}_Y(y) + \int_{\mu_1^*}^{\mu_2^*} d\mathbb{P}_Y(y) = 0 \\
\implies &   \mathcal{P}(\mu_1^* < Y < \mu_2^*) = \alpha 
\end{split}
\end{equation*}
\end{proof}


\begin{customthm}{3.2}[RQR with Finite Samples]\label{thm:finiteRQR}
 For any random variable $Y$ associated with an input $x$, we consider N realizations of this random variable : $\{y_i\}_{i=1,N}$. $ \forall \, \alpha \in [0,1]$ such that $\alpha \cdot N \in \mathbb{N},$ 
\begin{multline*}
(\mu_1^*(x),\mu_2^*(x)) = \argmin_{\mu_1,\mu_2}{\{ \sum_{i=1}^N\mathcal{L}^{\textup{RQR}}_{\alpha}((\mu_1,\mu_2), x, y_i)\}}\IfRestatedTF{}{\\} \implies  \frac{1}{N}\sum_{i=1}^N\mathbb{I}_{y_i \in [\mu_1*,\mu_2*]} = \alpha.  \\
\end{multline*}
\end{customthm}

\begin{proof}
We rewrite our sum as an integral with discrete density $\mathbb{I}_{y\in\{y_i\}}dy$ :

\begin{equation*}
    \sum_{i=1}^N\mathcal{L}^{\text{RQR}}_{\alpha}((\mu_1,\mu_2), x, y_i) =  \alpha \int_{-\infty}^{\infty} (y - \mu_1)(y - \mu_2)\mathbb{I}_{y\in\{y_i\}}dy - \int_{\mu_1}^{\mu_2} (y - \mu_1)(y - \mu_2)\mathbb{I}_{y\in\{y_i\}}dy
\end{equation*}

We find the expression of its partial derivatives with respect to the bounds:

\begin{equation*}
\begin{split}
\frac{\partial \sum_{i=1}^N\mathcal{L}^{\text{RQR}}_{\alpha}((\mu_1,\mu_2), x, y_i)}{\partial \mu_1} & =  -\alpha \int_{-\infty}^{\infty} (y - \mu_2)\mathbb{I}_{y\in\{y_i\}}dy + \int_{\mu_1}^{\mu_2} (y - \mu_2)\mathbb{I}_{y\in\{y_i\}}dy \\
\frac{\partial \sum_{i=1}^N\mathcal{L}^{\text{RQR}}_{\alpha}((\mu_1,\mu_2), x, y_i)}{\partial \mu_2} & =-\alpha \int_{-\infty}^{\infty} (y - \mu_1)\mathbb{I}_{y\in\{y_i\}}dy + \int_{\mu_1}^{\mu_2} (y - \mu_1)\mathbb{I}_{y\in\{y_i\}}dy
\end{split}
\end{equation*}

At the minimum, the gradient of the expected loss is null. Thus, if there is a minimum at the point $(\mu_1^*,\mu_2^*)$ with $\mu_2^* > \mu_1^*$,
\begin{equation*}
\begin{split}
& \frac{\partial \sum_{i=1}^N\mathcal{L}^{\text{RQR}}_{\alpha}((\mu_1,\mu_2), x, y_i)}{\partial \mu_1}\Bigr|_{\mu_1^*,\mu_2^*} - \frac{\partial\sum_{i=1}^N\mathcal{L}^{\text{RQR}}_{\alpha}((\mu_1,\mu_2), x, y_i)}{\partial \mu_2}\Bigr|_{\mu_1^*,\mu_2^*} = 0 \\
\implies & -\alpha \int_{-\infty}^{\infty} (\mu_1^* - \mu_2^*)\mathbb{I}_{y\in\{y_i\}}dy + \int_{\mu_1^*}^{\mu_2^*} (\mu_1^*  - \mu_2^*)\mathbb{I}_{y\in\{y_i\}}dy = 0\\
\implies & -\alpha (\mu_1^* - \mu_2^*) \sum_{i=1}^N 1 + (\mu_1^* - \mu_2^*) \sum_{y_i\in [\mu_1^*,\mu_2^*]}1 =0\\
\implies & \frac{1}{N}\sum_{i=1}^N\mathbb{I}_{y_i \in [\mu_1*,\mu_2*]} = \alpha 
\end{split}
\end{equation*}
\end{proof}

\begin{customthm}{3.3}[RQR is Unbiased and Consistent]
 For any random variable $Y$ associated with an input $x$, we consider N realizations of this random variable : $\{y_i\}_{i=1}^{N}$. $(\mu_{1,N},\mu_{2,N})$ are the bounds of our estimator trained on these N samples. $ \forall \, \alpha \in [0,1]$, we name $Q_N$ the absolute true miscoverage of our estimator  
 \begin{equation*}
     Q_N = |\int_{\mu_{1,N}}^{\mu_{2,N}}dPy - \alpha|
 \end{equation*}

Then $E[Q_N] = 0$ and the variance of $Q_N$ is bounded by $\frac{1}{4N}$.
\end{customthm}
\begin{proof}
We can decompose $Q_N$ in two terms :

\begin{equation*}
    Q_N \leq |\int_{\mu_{1,N}}^{\mu_{2,N}}dPy - \frac{1}{N}\sum_{i=1}^N \mathbb{I}_{y_i \in [\mu_{1,N},\mu_{2,N}]}| + |\frac{1}{N}\sum_{i=1}^N \mathbb{I}_{y_i \in [\mu_{1,N},\mu_{2,N}]} - \alpha|
\end{equation*}

The theorem \ref{thm:finiteRQR} proves that $\forall N \in \mathbb{N}$ the second term is null.

Then, the first term is simply the difference between an integral and its Monte Carlo estimator. Hence, we can derive well-known results about its mean $\mathbb{E}[Q_N]$ and variance $\mathbb{V}[Q_N]$.

\begin{equation*}
\begin{split}
    \mathbb{E}[\int_{\mu_{1,N}}^{\mu_{2,N}}dPy - \frac{1}{N}\sum_{i=1}^N \mathbb{I}_{y_i \in [\mu_{1,N},\mu_{2,N}]}] &= \int_{\mu_{1,N}}^{\mu_{2,N}}dPy - \frac{1}{N}\sum_{i=1}^N \mathbb{E}[\mathbb{I}_{y_i \in [\mu_{1,N},\mu_{2,N}]}] \\
    &= \int_{\mu_{1,N}}^{\mu_{2,N}}dPy - \frac{1}{N}\sum_{i=1}^N\int_{\mu_{1,N}}^{\mu_{2,N}}dPy \\
    &=0
\end{split}
\end{equation*}
Hence, $\mathbb{E}[Q_N] \leq \mathbb{E}[\int_{\mu_{1,N}}^{\mu_{2,N}}dPy - \frac{1}{N}\sum_{i=1}^N \mathbb{I}_{y_i \in [\mu_{1,N},\mu_{2,N}]}] =0$.

\begin{equation*}
    \mathbb{V}[\int_{\mu_{1,N}}^{\mu_{2,N}}dPy - \frac{1}{N}\sum_{i=1}^N \mathbb{I}_{y_i \in [\mu_{1,N},\mu_{2,N}]}] = \frac{1}{N}\mathbb{V}[\mathbb{I}_{Y \in [\mu_{1,N},\mu_{2,N}]}] \\
\end{equation*}

Thus, $\mathbb{V}[Q_N] \leq \frac{1}{N}\mathbb{V}[\mathbb{I}_{Y \in [\mu_{1,N},\mu_{2,N}]}]$
\\
Moreover, $\mathbb{V}[\mathbb{I}_{Y \in [\mu_{1,N},\mu_{2,N}]}] = P(Y \in [\mu_{1,N},\mu_{2,N}])(1-P(Y \in [\mu_{1,N},\mu_{2,N}]))$ and the study of the function $f:x\mapsto x(1-x)$ for $x \in [0,1]$ shows that this function is bounded by $\frac{1}{4}$.
\\
Thus,  $\mathbb{V}[Q_N] \leq \mathbb{V}[\int_{\mu_{1,N}}^{\mu_{2,N}}dPy - \frac{1}{N}\sum_{i=1}^N \mathbb{I}_{y_i \in [\mu_{1,N},\mu_{2,N}]}] \leq \frac{1}{4N}$
\end{proof}

\begin{customthm}{3.4}[RQR-W In-sample Coverage]
For any random  variable $Y$ associated with an input x, $\forall \, \alpha \in [0,1],$
\begin{multline*}
(\mu_1^*(x),\mu_2^*(x)) = \argmin_{\mu_1,\mu_2}{\{ \mathbb{E}_Y( (\mu_1,\mu_2),x,Y))\}}\implies \mathbb{P}(\mu_1^*(x) < Y < \mu_2^*(x)) = \alpha 
\end{multline*}
\end{customthm}
\begin{proof}
Similarly, we are starting by rewriting the RQR-W loss with indicator functions.
\begin{equation*}
    \begin{split}
        \text{RQR-W}_{\alpha}(\mu_1,\mu_2,y) &  = \text{RQR}_{\alpha}(\mu_1,\mu_2,y) + \frac{\lambda(\mu_2-\mu_1)^2}{2}\\
         & = (y-\mu_1)(y-\mu_2)(\alpha-\mathbb{I}_{y \in [\mu_1,\mu_2]}) + \frac{\lambda(\mu_2-\mu_1)^2}{2}
    \end{split}
\end{equation*}

Then, we consider the expectation of the loss :
\begin{equation*}
\begin{multlined}
\mathbb{E}(\text{RQR-W}_{\alpha}(\mu_1,\mu_2,Y)) = \\
\alpha \int_{-\infty}^{\infty} (y - \mu_1)(y - \mu_2)d\mathbb{P}_Y(y) - \int_{\mu_1}^{\mu_2} (y - \mu_1)(y - \mu_2)d\mathbb{P}_Y(y) + \frac{\lambda(\mu_2-\mu_1)^2}{2}
\end{multlined}
\end{equation*}

We find the expression of the partial derivatives with respect to the bounds:

\begin{equation*}
\begin{split}
\frac{\partial \mathbb{E}(\text{RQR-W}_{\alpha})}{\partial \mu_1} & =  -\alpha \int_{-\infty}^{\infty} (y - \mu_2)d\mathbb{P}_Y(y) + \int_{\mu_1}^{\mu_2} (y - \mu_2)d\mathbb{P}_Y(y) - \lambda(\mu_2-\mu_1)\\
\frac{\partial \mathbb{E}(\text{RQR-W}_{\alpha})}{\partial \mu_2} & = -\alpha \int_{-\infty}^{\infty} (y - \mu_1)d\mathbb{P}_Y(y) + \int_{\mu_1}^{\mu_2} (y - \mu_1)d\mathbb{P}_Y(y) + \lambda(\mu_2-\mu_1)
\end{split}
\end{equation*}

At the minimum, the gradient of the expected loss is null. Thus, if there is a minimum at the point $(\mu_1^*,\mu_2^*)$ with $\mu_2^* > \mu_1^*$,
\begin{equation*}
\begin{split}
& \frac{\partial \mathbb{E}(\text{RQR-W}_{\alpha})}{\partial \mu_1}\Bigr|_{\mu_1^*,\mu_2^*} - \frac{\partial \mathbb{E}(\text{RQR-W}_{\alpha})}{\partial \mu_2}\Bigr|_{\mu_1^*,\mu_2^*} = 0 \\
\implies & -\alpha \int_{-\infty}^{\infty} (\mu_1^* - \mu_2^*)d\mathbb{P}_Y(y) + \int_{\mu_1^*}^{\mu_2^*} (\mu_1^*  - \mu_2^*)d\mathbb{P}_Y(y) + 2\lambda(\mu_1^* - \mu_2^*)  = 0\\
\implies & -\alpha (\mu_1^* - \mu_2^*)\int_{-\infty}^{\infty}d\mathbb{P}_Y(y) + (\mu_1^*  - \mu_2^*) \int_{\mu_1^*}^{\mu_2^*} d\mathbb{P}_Y(y) + 2\lambda(\mu_1^* - \mu_2^*) = 0 \\
\implies & -\alpha \int_{-\infty}^{\infty}d\mathbb{P}_Y(y) + \int_{\mu_1^*}^{\mu_2^*} d\mathbb{P}_Y(y) + 2\lambda = 0 \\
\implies &   \mathcal{P}(\mu_1^* < Y < \mu_2^*) = \alpha - 2\lambda
\end{split}
\end{equation*}

Then, as $\alpha$ is a constant, we can replace it with a corrected term. When choosing $\hat{\alpha} = \alpha + 2\lambda$, we obtain 

\begin{equation*}
    (\mu_1^*,\mu_2^*) = \argmin_{\mu_1,\mu_2}{\{ \mathbb{E}(\text{RQR-W}_{\hat{\alpha}}(\mu_1,\mu_2,Y))\}} 
    \implies \mathcal{P}(\mu_1^* < Y < \mu_2^*) = \hat{\alpha} -2\lambda = \alpha
\end{equation*}

\end{proof}

\begin{customthm}{3.5}[RQR-W with Finite Samples] \label{thm:finiteRQRW}
  For any random variable $Y$ associated with an input $x$, we consider N realizations of this random variable: $\{y_i\}_{i=1,N}$. $ \forall \, \alpha \in [0,1]$ such that $\alpha \cdot N \in \mathbb{N}$,
  \begin{equation*}
    (\mu_1^*(x),\mu_2^*(x)) = \argmin_{\mu_1,\mu_2}{\{ \sum_{i=1}^N\mathcal{L}^{\textup{RQR-W}}_{\alpha+ 2\lambda}((\mu_1,\mu_2), x, y_i)\}} \implies  \frac{1}{N}\sum_{i=1}^N\mathbb{I}_{y_i \in [\mu_1*,\mu_2*]} = \alpha  \\
  \end{equation*}
\end{customthm}

\begin{proof}

We rewrite our sum as an integral with discrete density $\mathbb{I}_{y\in\{y_i\}}dy$ :
\begin{equation*}
    \sum_{i=1}^N\mathcal{L}^{\text{RQR-W}}_{\alpha + 2\lambda}((\mu_1,\mu_2), x, y_i) =  \alpha \int_{-\infty}^{\infty} (y - \mu_1)(y - \mu_2)\mathbb{I}_{y\in\{y_i\}}dy - \int_{\mu_1}^{\mu_2} (y - \mu_1)(y - \mu_2)\mathbb{I}_{y\in\{y_i\}}dy + N\frac{\lambda(\mu_2-\mu_1)^2}{2}
\end{equation*}

We find the expression of its partial derivatives with respect to the bounds:

\begin{equation*}
\begin{split}
\frac{\partial \sum_{i=1}^N\mathcal{L}^{\text{QFR}}_{\alpha+2\lambda}((\mu_1,\mu_2), x, y_i)}{\partial \mu_1} & =  -(\alpha+2\lambda) \int_{-\infty}^{\infty} (y - \mu_2)\mathbb{I}_{y\in\{y_i\}}dy + \int_{\mu_1}^{\mu_2} (y - \mu_2)\mathbb{I}_{y\in\{y_i\}}dy -N\lambda(\mu_2 - \mu_1) \\
\frac{\partial \sum_{i=1}^N\mathcal{L}^{\text{QFR}}_{\alpha}((\mu_1,\mu_2), x, y_i)}{\partial \mu_2} & =-(\alpha+2\lambda) \int_{-\infty}^{\infty} (y - \mu_1)\mathbb{I}_{y\in\{y_i\}}dy + \int_{\mu_1}^{\mu_2} (y - \mu_1)\mathbb{I}_{y\in\{y_i\}}dy +N\lambda(\mu_2 - \mu_1)
\end{split}
\end{equation*}

At the minimum, the gradient of the expected loss is null. Thus, if there is a minimum at the point $(\mu_1^*,\mu_2^*)$ with $\mu_2^* > \mu_1^*$,
\begin{equation*}
\begin{split}
& \frac{\partial \sum_{i=1}^N\mathcal{L}^{\text{RQR-W}}_{(\alpha+2\lambda)}((\mu_1,\mu_2), x, y_i)}{\partial \mu_1}\Bigr|_{\mu_1^*,\mu_2^*} - \frac{\partial\sum_{i=1}^N\mathcal{L}^{\text{RQR-W}}_{\alpha+2\lambda}((\mu_1,\mu_2), x, y_i)}{\partial \mu_2}\Bigr|_{\mu_1^*,\mu_2^*} = 0 \\
\implies & -(\alpha+2\lambda) \int_{-\infty}^{\infty} (\mu_1^* - \mu_2^*)\mathbb{I}_{y\in\{y_i\}}dy + \int_{\mu_1^*}^{\mu_2^*} (\mu_1^*  - \mu_2^*)\mathbb{I}_{y\in\{y_i\}}dy  + 2N\lambda(\mu_2^* - \mu_1^*)= 0\\
\implies & -(\alpha+2\lambda)(\mu_1^* - \mu_2^*) \sum_{i=1}^N 1 + (\mu_1^* - \mu_2^*) \sum_{y_i\in [\mu_1^*,\mu_2^*]}1 + 2N\lambda(\mu_2^* - \mu_1^*)=0\\
\implies & \frac{1}{N}\sum_{i=1}^N\mathbb{I}_{y_i \in [\mu_1*,\mu_2*]} = \alpha 
\end{split}
\end{equation*}
\end{proof}

\begin{customthm}{3.6}[RQR-W is Unbiased and Consistant] \label{thm:RQRWasymptoticdata}
 For any random variable $Y$ associated with an input $x$, we consider N realizations of this random variable : $\{y_i\}_{i=1,N}$. $(\mu_{1,N},\mu_{2,N})$ are the bounds of our estimator trained on these N samples. $ \forall \, \alpha \in [0,1]$, we name $Q_N$ the absolute true miscoverage of our estimator  

 \begin{equation*}
     Q_N = |\int_{\mu_{1,N}}^{\mu_{2,N}}dPy - \alpha|
 \end{equation*}

Then $E[Q_N] = 0$ and the variance of $Q_N$ is bounded by $\frac{1}{4N}$

\end{customthm}
\begin{proof}
We can decompose $Q_N$ in two terms :

\begin{equation*}
    Q_N \leq |\int_{\mu_{1,N}}^{\mu_{2,N}}dPy - \frac{1}{N}\sum_{i=1}^N \mathbb{I}_{y_i \in [\mu_{1,N},\mu_{2,N}]}| + |\frac{1}{N}\sum_{i=1}^N \mathbb{I}_{y_i \in [\mu_{1,N},\mu_{2,N}]} - \alpha|
\end{equation*}

The theorem \ref{thm:finiteRQRW} proves that $\forall N \in \mathbb{N}$ the second term is null.

Then, the first term is simply the difference between an integral and its Monte Carlo estimator. Hence, we can derive well-known results about its mean $\mathbb{E}[Q_N]$ and variance $\mathbb{V}[Q_N]$.

\begin{equation*}
\begin{split}
    \mathbb{E}[\int_{\mu_{1,N}}^{\mu_{2,N}}dPy - \frac{1}{N}\sum_{i=1}^N \mathbb{I}_{y_i \in [\mu_{1,N},\mu_{2,N}]}] &= \int_{\mu_{1,N}}^{\mu_{2,N}}dPy - \frac{1}{N}\sum_{i=1}^N \mathbb{E}[\mathbb{I}_{y_i \in [\mu_{1,N},\mu_{2,N}]}] \\
    &= \int_{\mu_{1,N}}^{\mu_{2,N}}dPy - \frac{1}{N}\sum_{i=1}^N\int_{\mu_{1,N}}^{\mu_{2,N}}dPy \\
    &=0
\end{split}
\end{equation*}
Hence, $\mathbb{E}[Q_N] \leq \mathbb{E}[\int_{\mu_{1,N}}^{\mu_{2,N}}dPy - \frac{1}{N}\sum_{i=1}^N \mathbb{I}_{y_i \in [\mu_{1,N},\mu_{2,N}]}] =0$.

\begin{equation*}
    \mathbb{V}[\int_{\mu_{1,N}}^{\mu_{2,N}}dPy - \frac{1}{N}\sum_{i=1}^N \mathbb{I}_{y_i \in [\mu_{1,N},\mu_{2,N}]}] = \frac{1}{N}\mathbb{V}[\mathbb{I}_{Y \in [\mu_{1,N},\mu_{2,N}]}] \\
\end{equation*}

Thus, $\mathbb{V}[Q_N] \leq \frac{1}{N}\mathbb{V}[\mathbb{I}_{Y \in [\mu_{1,N},\mu_{2,N}]}]$
\\
Moreover, $\mathbb{V}[\mathbb{I}_{Y \in [\mu_{1,N},\mu_{2,N}]}] = P(Y \in [\mu_{1,N},\mu_{2,N}])(1-P(Y \in [\mu_{1,N},\mu_{2,N}]))$ and the study of the function $f:x\mapsto x(1-x)$ for $x \in [0,1]$ shows that this function is bounded by $\frac{1}{4}$.
\\
Thus,  $\mathbb{V}[Q_N] \leq \mathbb{V}[\int_{\mu_{1,N}}^{\mu_{2,N}}dPy - \frac{1}{N}\sum_{i=1}^N \mathbb{I}_{y_i \in [\mu_{1,N},\mu_{2,N}]}] \leq \frac{1}{4N}$

\end{proof}

\begin{customprop}{3.7}[Existence and Uniqueness of Solution]
$\mu_1^{min}$ and $\mu_2^{max}$ denote the bounds of our optimization problem. For a target distribution $Y$ with a cumulative distribution function that is k-Lipschitz continuous with $k < 1+\frac{\alpha}{\mu_2^{max}-\mu_1^{min}}$, when $\lambda > max(0,\int_{\mu_1^{min}}^{\mu_2^{max}} d\mathbb{P}_Y(y) - \alpha)$, the minimum of $\mathcal{L}^{\text{RQR-W}}_{\alpha+2\lambda}$ exists and is unique.
\end{customprop}
\begin{proof}
We study the function $\mu_1,\mu_2  \mapsto \mathbb{E}(\mathcal{L}^{\text{RQR-W}}_{\alpha}((\mu_1,\mu_2),Y)$ in a closed subset of $\mathbb{R}^2$ where $\mu_1 < \mu_2$. We named this subset $\mathbb{D}$. 
On this closed subset of the space to say, it exists $\mu_2^{max}$ and $\mu_1^{min}$ such that $\forall (\mu_1,\mu_2) \in \mathbb{D}$  $\mu_2 < \mu_2^{max} $ and $\mu_1 > \mu_1^{min}$. 

Moreover, we assume that Y can be associated with a probability density function $d\mathbb{P}_Y$ and that its cumulative distribution function is k-Lipschitz continuous with $k < 1 + \frac{\alpha}{\mu_2^{max}-\mu_1^{min}} $
\begin{equation*}
\forall (\mu_2,\mu_1) \in [\mu_1^{min},\mu_2^{max}]^2 ~|\int_{\mu_1}^{\mu_2}d\mathbb{P}_Y(y)| < (1 + \frac{\alpha}{\mu_2^{max}-\mu_1^{min}})|(\mu_2-\mu_1)| 
\end{equation*}

The non-negativity of the studied function gives us the existence of the minimum.

To demonstrate the uniqueness of the minimum, we show that the eigenvalues of the Hessian matrix are positive. We start by computing the gradient of the expected loss :

\begin{equation*}
\begin{split}
    \nabla\mathbb{E}(\mathcal{L}^{\text{RQR-W}}_{\alpha}) = &
\begin{bmatrix}
    \frac{\partial \mathbb{E}(\mathcal{L}^{\text{RQR-W}}_{\alpha})}{\partial \mu_1} \\
    \frac{\partial \mathbb{E}(\mathcal{L}^{\text{RQR-W}}_{\alpha})}{\partial \mu_2} 
\end{bmatrix} \\
    = &
    \begin{bmatrix}
    -\alpha \int_{-\infty}^{\infty} (y - \mu_2)d\mathbb{P}_Y(y) + \int_{\mu_1}^{\mu_2} (y - \mu_2)d\mathbb{P}_Y(y)  - \lambda(\mu_2-\mu_1) \\
    -\alpha \int_{-\infty}^{\infty} (y - \mu_1)d\mathbb{P}_Y(y) + \int_{\mu_1}^{\mu_2} (y - \mu_1)d\mathbb{P}_Y(y) + \lambda(\mu_2-\mu_1)
    \end{bmatrix}
    \end{split}
\end{equation*}

Then, we compute the Hessian matrix:
\begin{equation*}
    \nabla^2\mathbb{E}(\mathcal{L}^{\text{RQR-W}}_{\alpha}) = 
    \begin{bmatrix}
       \mu_2 - \mu_1 + \lambda & \alpha - \int_{\mu_1}^{\mu_2} d\mathbb{P}_Y(y) - \lambda\\
       \alpha - \int_{\mu_1}^{\mu_2} d\mathbb{P}_Y(y) - \lambda & \mu_2 - \mu_1 + \lambda
    \end{bmatrix}
\end{equation*}

The eigenvalues of the Hessian matrix are given by  $\lambda_{\pm} = \mu_2 - \mu_1 + \lambda \pm | \alpha - \int_{\mu_1}^{\mu_2}d\mathbb{P}_Y(y) - \lambda |$

From that, we get that $\lambda > \int_{\mu_1^{min}}^{\mu_2^{max}} d\mathbb{P}_Y(y) - \alpha \implies \forall (\mu_1,\mu_2) \in \mathbb{D} ~\lambda > \int_{\mu_1}^{\mu_2} d\mathbb{P}_Y(y) - \alpha~(1)$ because the probability density function $d\mathbb{P}_Y$ is positive. 

Additionally, we use the assumption of the k-Lipschitz continuity of the CDF and we obtain the following inequality : 

\begin{equation*}
    \forall (\mu_1,\mu_2) \in \mathbb{D} ~\mu_2 - \mu_1  \geq \int_{\mu_1}^{\mu_2}(1 + \frac{\alpha}{\mu_2-\mu_1}) - \alpha  \geq \int_{\mu_1}^{\mu_2}(1 + \frac{\alpha}{\mu_2^{max}-\mu_1^{min}}) - \alpha \geq \int_{\mu_1}^{\mu_2}d\mathbb{P}_Y - \alpha
\end{equation*}

Therefore, under the condition $\lambda > \max(0,\int_{\mu_1^{min}}^{\mu_2^{max}} d\mathbb{P}_Y(y) - \alpha$), we obtain the positiveness of  $\lambda_-$ : 

\begin{equation*}
    \begin{split}
        & \lambda_- =  \mu_2 - \mu_1 + \lambda - | \alpha - \int_{\mu_1}^{\mu_2}d\mathbb{P}_Y(y) - \lambda | = \mu_2 - \mu_1 + \lambda - (\lambda - \alpha + \int_{\mu_1}^{\mu_2}d\mathbb{P}_Y(y)) ~(1) \\
        \implies & \lambda_- = \mu_2 - \mu_1 - (-\alpha + \int_{\mu_1}^{\mu_2}d\mathbb{P}_Y(y)) > 0
    \end{split}
\end{equation*}
        
The first eigenvalue $\lambda_+$ is obviously non-negative, thus both eigenvalues are non-negative which means that the Hessian matrix is semi-definite positive. 

In conclusion, when the condition $\lambda > \max(0,\int_{\mu_1^{min}}^{\mu_2^{max}}d\mathbb{P}_Y(y) -\alpha) $ is respected, our optimal interval prediction loss is convex. Hence, it has a unique minimum.

Both the existence and the uniqueness of the minimum have been proven. 
\end{proof}

\begin{customthm}{3.8}[Validity of RQR-O - A variation of the validity of orthogonal quantile regression theorem from \citet{feldman2021improving}] \label{thm:RQR-Ocoverage}
Suppose $Y |X = x$ follows a continuous distribution for each $x \in \mathcal{X}$, and suppose that $\mu_1(X), \mu_2(X) \in \mathcal{F}$.
Consider the infinite-data version of the RQR-O optimization :
\begin{equation*}
\argmin_{\mu_1,\mu_2 \in \mathbb{F}}{\{ \mathbb{E}(\mathcal{L}^{\text{RQR}}_{\alpha}((\mu_1,\mu_2),X,Y) + \gamma\mathbb{R}(w,m))\}} 
\end{equation*}
Then, true conditional intervals with $\alpha$ coverage are solutions to the above optimization problem.
\end{customthm}

\begin{proof}
We note $m = \mathbb{I}_{y \in [\mu_1,\mu_2]}$ the coverage function and $w$ = $|\mu_2-\mu_1|$ the interval length function. We consider the true conditional intervals that satisfied $\mathbb{P}(\mu_1(X)< Y < \mu_2(X)|X=x) = \alpha$. \citet{feldman2021improving} has shown that for such intervals, the coverage and interval length functions are independent. 

Therefore, similarly to vanilla QR or the Winkler score, these true conditional intervals are a solution to the RQR problem. Moreover, by definition, the independence of $w$ and $m$ fixes the Pearson correlation (or the HSIC score) of these functions to 0. Thus, the orthogonal penalty term that is based on these metrics is also minimized. 
Hence, the true conditional intervals with $\alpha$ coverage are a solution to the RQR-0 optimization problem. 
\end{proof}


\section{Experimental Details} \label{app:experimentdetails}
We follow standard preprocessing on all datasets with features standardized such that they have zero mean and unit variance and targets divided by their mean. All experiments are repeated over 10 random seeds with means and standard errors of means reported throughout. We train two-layer neural networks with 64 hidden units and ReLU activations throughout (consistent with \citet{feldman2021improving}). We also use the Adam optimizer \citep{kingma2014adam}. 

In the benchmarking of RQR-W, we applied the following experimental design. A training-validation-testing split with a ratio [0.6,0.2,0.2] is applied. Then we perform a hyperparameter grid search for all methods. Each combination is first fit to the training data with the model selection based on evaluations on the validation set. For a given hyperparameter combination, the best model is selected based on the epoch that achieves the best interval length (such that target coverage is achieved). The reported values are the evaluations on the test set. All methods evaluated follow this protocol. The grid search considers the following hyperparameters: dropout probability $\in \{0.1,0.2,0.3\}$, learning rate $\in  [0.1,0.05,0.01,0.005,0.001,0.0005,0.0001]$, and regularization coefficient $\in [0.01,0.1,1,5,10,20,30,40,50]$. Other hyperparameters are fixed: number of epochs $400$, batch size $10000$.

In the benchmarking of RQR-O, we followed the exact procedure and used the implementation of the OQR baseline in \citet{feldman2021improving}. In what follows, we will describe this procedure. A training-validation-testing split with a ratio of 0.4 for testing and a further split of 0.9-0.1 for training-validation is applied. The default hyperparameters are used for both methods with only the regularization coefficient tuned. Specifically, it is set to 1 and then decreased in increments following $[1, 0.5, 0.1, 0.05, \ldots]$ until the desired coverage is achieved. Both the RQR-O and OQR hyperparameters are - learning rate: 1e-3, maximum number of epochs: 10000, dropout probability: 0, and batch size: 1024. Early stopping patience is set to 200 epochs.

\begin{table}[h]
\caption{Summary statistics of the standard benchmark datasets for quantile regression.}
\label{tab:summarystats}
\centering
\begin{tabular}{ccccc} \toprule
{Dataset} & {Mean} & {Variance} & {Skewness} & {Kurtosis} \\ 
\midrule
{concrete} & 35.82 & 279.08 & 0.42 & -0.32 \\
{wine}  & 5.64 & 0.65 & 0.22 & 0.29 \\
{yacht}  & 10.50 & 229.84 & 1.75 & 2.00 \\
{energy}  & 22.31 & 101.81 & 0.36 & -1.25 \\
{kin8nm}  & 0.71 & 0.07 & 0.09 & -0.53 \\
{naval}  & 0.99 & 0.00 & -0.00 & -1.20 \\
{power}  & 454.37 & 291.28 & 0.31 & -1.05 \\
{boston}  & 22.53 & 84.59 & 1.10 & 1.47 \\
{protein} & 7.75 & 37.43 & 0.57 & -1.14 \\
\bottomrule
\end{tabular}
\end{table}

\section{Formal Description of Evaluation Metrics} \label{app:metrics}
To ensure this work is self-contained, in this section we include a formal description of the evaluation metrics introduced in \Cref{sec:background} and used in \Cref{sec:experiments}. We implement these metrics as described in the referenced works. In all cases, we evaluate on some evaluation dataset $\mathcal{D}$ consisting of $N$ input/target pairs $(\mathbf{x}, y)$. 

\begin{definition}[Prediction Interval Coverage Probability (PICP) \citep{tagasovska2019single}]
Defined as the number of true observations falling inside the estimated prediction interval, this is calculated as
\begin{equation*}
    PICP \vcentcolon= \frac{1}{N} \sum_{(\mathbf{x}, y) \sim \mathcal{D}} \mathbb{I}_{\mu_1 \leq y \leq \mu_2}.
\end{equation*}
\end{definition}

\begin{definition}[Mean Prediction Interval Width (MPIW) \citep{tagasovska2019single}]
Defined as the average interval width across the evaluation dataset, this is calculated as
\begin{equation*}
    MPIW \vcentcolon= \frac{1}{N} \sum_{(\mathbf{x}, y) \sim \mathcal{D}} |{\mu_2 - \mu_1}|.
\end{equation*}
\end{definition}

\begin{definition}[Width-coverage Pearson's Correlation (WCPC) \citep{feldman2021improving}] \label{def:WCPC}
Denoting $\mathbf{w}$ as the vector of interval widths where ${w}_i = |\mu_2(\mathbf{x}_i) - \mu_1(\mathbf{x}_i)|$ and $\mathbf{m}$ as the indicator vector of coverage events where ${m}_i = \mathbb{I}_{y_i \in [\mu_1(\mathbf{x}_i), \mu_2(\mathbf{x}_i)]}$ for $ i \in \{1, 2, \ldots, N\}$, then we define
\begin{equation*}
    WCPC \vcentcolon= \left|\frac{\text{Cov}(\mathbf{w}, \mathbf{m})}{\text{Var}(\mathbf{w})\text{Var}(\mathbf{m})}\right|.
\end{equation*}
\end{definition}

\begin{definition}[Hilbert-Schmidt Independence
Criterion (HSIC)  \citep{greenfeld2020robust, feldman2021improving}]
In this definition, upper case bold letters denote matrices. Using the same definitions of $\mathbf{w}$ and $\mathbf{m}$ from \Cref{def:WCPC}, the coverage kernel matrix is given by $\mathbf{R}_{i,j} = k(\mathbf{m}_i, \mathbf{m}_j)$ and the width kernel is given by $\mathbf{K}_{i,j} = k(\mathbf{w}_i, \mathbf{w}_j)$ where $k$ denotes the Gaussian kernel. We also introduce a centering matrix $\mathbf{H}_{i,j} = \delta_{i,j} - \frac{1}{N}$ where $\delta_{i,j} = 1$ if $i = j$ and $0$ otherwise.  Then we can calculate HSIC as
\begin{equation*} 
    HSIC \vcentcolon=  \sqrt{\frac{tr(\mathbf{KHRH})}{(N - 1)^2}} 
\end{equation*}
\end{definition}

\section{Crossing Bounds} \label{sec:crossingbounds}
A well-known limitation of the quantile regression approach to constructing prediction intervals is any estimation error of the population level intervals can result in \textit{crossing bounds} where the upper bound of the interval falls \textit{below} the lower bound of the interval (see e.g. \citet{brando2022deep, park2022learning}). Apart from being a conceptual limitation, this can also affect the users' trust in the system. A key advantage of directly estimating the interval as proposed in this work is that the raw model outputs are not strictly associated with being a particular bound. In other words, our objective is invariant to permutations of the upper and lower bounds, and simply sets the lower bound to be the minimum and the upper bound to be the maximum of the two model outputs. Therefore, \textit{which} output neuron acts as either bound can even change on a per-example basis. Formally, this can be observed in the $\kappa = (y - \mu_1)(y - \mu_2)$ term in our loss function (\Cref{eqn:RQRbaseobjective}) which is clearly invariant to permutations between $\mu_1$ and $\mu_2$. In contrast, crossing bounds in the case of quantile regression will result in miscoverage if the bounds are permuted to avoid a negative interval. This may be considered a distinct advantage of the direct interval prediction approach which sidesteps the conceptual issue of crossing quantiles which occurs due to the independent estimation of the upper and lower bounds in quantile-based approaches.

Recent work in \citet{brando2022deep} has proposed methods for preventing crossing bounds in the case of quantile regression by using Chebyshev polynomials to add additional constraints into the the objective. In what follows we compare our proposed RQR-W method to this non-crossing quantile method using the same experimental setup as in \Cref{tab:widthresults}. As this method is a quantile estimation method and does not prescribe how to construct intervals we include two approaches (a) ``narrowest'' where we bin search the narrowest interval at each new inference and (b) ``symmetric'' where we use the symmetric quantiles (0.05, 0.95). The results are included in \Cref{tab:noncrossing}.

\begin{table}[h]
\centering
\caption{\textbf{Comparison between RQR-W vs Deep Non-crossing Quantiles}. We present two different versions of the deep non-crossing quantiles method: "Narrowest" and "Symmetric" (see text for details).  For the 9 datasets, we display interval width and the coverage level achieved in parentheses. }
\label{tab:trunc_gaussian}
\begin{tblr}{
  width = \linewidth,
  colspec = {Q[179]Q[242]Q[242]Q[244]},
  cells = {c},
  hline{1,11} = {-}{0.08em},
  hline{2} = {1,3-4}{0.03em},
  hline{2} = {2}{},
}
Dataset & RQR-W & Narrowest & Symmetric\\
concrete & 0.43 (89.95) & 0.20 (57.33) & 0.30 (73.30)\\
boston & 0.51 (89.8) & 0.26 (64.90) & 0.36 (78.53)\\
naval & 0.01 (90.95) & 0.01 (97.43) & 0.02 (99.04)\\
energy & 0.13 (89.87) & 0.07 (50.39) & 0.14 (72.86)\\
wine & 0.35 (89.44) & 0.19 (62.56) & 0.28 (76.75)\\
power & 0.03 (90.73) & 0.02 (73.70) & 0.03 (89.83)\\
protein & 1.59 (90.07) & 1.13 (78.04) & 1.50 (89.26)\\
kin8nm & 0.34 (90.59) & 0.25 (70.96) & 0.38 (85.87)\\
yacht & 0.27 (93.23) & 0.29 (51.29) & 0.35 (61.61)
\end{tblr}
\label{tab:noncrossing}
\end{table}

In these results we note that, as with the previously evaluated methods for estimating all quantiles simultaneously, the non-crossing quantile approach consistently struggles to maintain coverage at test time. We seek solutions that minimize interval width \textit{for a fixed level of coverage} which is not achieved by this baseline which we attribute to two limitations. Firstly, when we select the narrowest possible interval from a set of possible intervals that each obtain marginal coverage, it is not a random choice and it is \textit{more} likely that we are selecting a particular interval in which coverage is not maintained. Therefore, analogous to the effect of \textit{overfitting} in the standard point prediction setting, we would expect that picking the narrowest intervals is likely to result in more miscoverage events than selecting a fixed interval (as is reflected in these results). Secondly, learning all quantiles simultaneously is a more challenging learning problem than simply learning two quantiles for a fixed capacity model. This results in poorer predictive performance at the task in hand.

\section{Motivation for Optimizing Interval Width \& Conditional Coverage} \label{sec:intervalwidth}

Although several methods might achieve a desired level of marginal coverage with their intervals, the task of finding a set of intervals that obtain such coverage is generally an underspecified problem with a potentially infinite number of admissible solutions. Given this, any interval-producing method is required to introduce some additional regularization (either implicit or explicit) in order to select among the possible solutions. In the case of quantile regression, the pair of symmetric quantiles are selected. However, given that other attributes are typically desired in these intervals (as illustrated by the large body of work that attempts to minimize interval width \citep{chung2021beyond, tagasovska2019single, pearce2018high} or maximize conditional coverage \citep{feldman2021improving, hunter2000quantile}), for many applications there are likely more preferable solutions than the symmetric quantiles found via quantile regression. This has been observed in real-world applications such as renewable energy forecasting where it has been noted that ``the probability distribution of the renewable energy source’s power output is generally skewed, thereby the width of central prediction intervals is often unnecessarily wide'' \citep{zhang2023contextual}.

One additional motivation for minimizing the width of intervals beyond the advantage of narrower predictive intervals is that the minimal width intervals generally lie in denser regions of the underlying probability distribution. Because we estimate these statistics using a sample from the population distribution we typically incur some variance in this estimate which is typically lower in the more dense regions of the probability space. To illustrate this, consider the distribution we presented in \Cref{fig:skewillustration} of the main text where the ground truth intervals are known. Now suppose we take a sample from this distribution (i.e. a training set) and estimate both the empirical minimum width intervals and the empirical quantiles. In both cases there is likely to be some estimation error. If we were to repeat this process multiple times we note that the variance around the estimates in lower-density regions (where the symmetric quantiles lie) is likely to be greater than in the high-density regions (where the minimum width intervals lie). As a consequence of this, quantile estimation methods may suffer from larger errors in estimating true quantiles (particularly in small sample regimes) resulting in more difficulties in obtaining exact coverage. We verify this claim empirically in \Cref{fig:boundvariance}.

\begin{figure}[h]
    \centering
    \includegraphics[width=\textwidth]{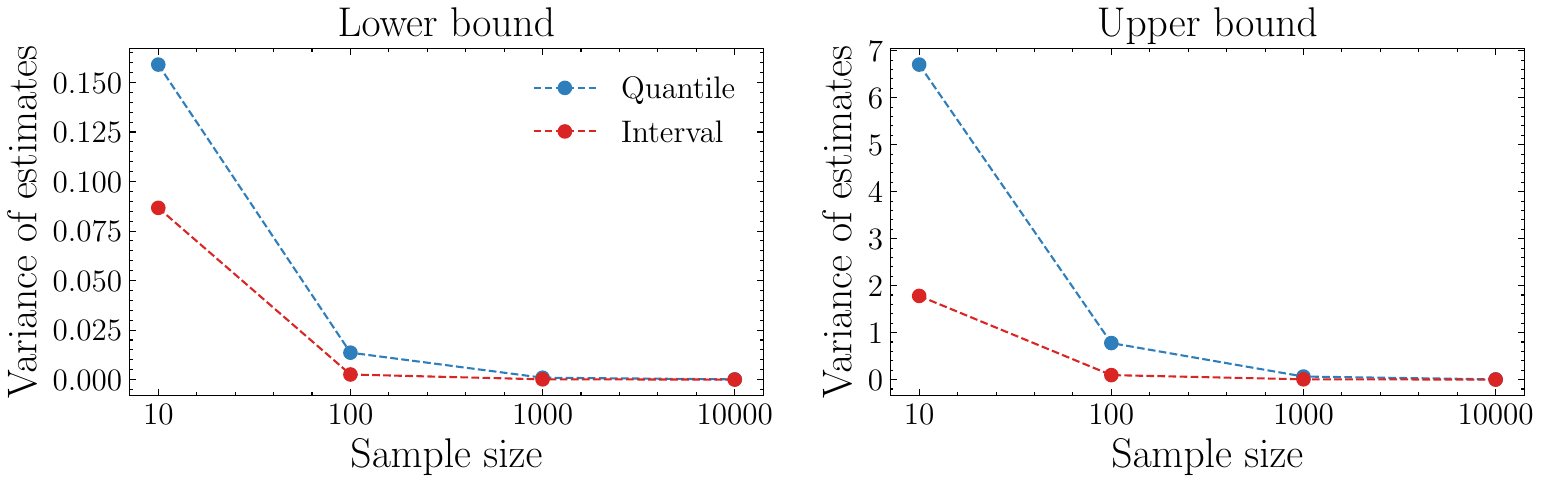}
    \caption{\textbf{Variance of the estimated bounds.} Estimating the minimal width interval rather than symmetric quantiles is also likely to result in lower variance estimates of the two bounds due to being estimated in more dense regions of the distribution. This is illustrated on the log-normal distribution example from \Cref{fig:skewillustration}. We take samples of various sizes from this distribution and estimate both the quantiles and the minimum width bounds. We find that the variance of these estimates is significantly larger when estimating the former.}
    \label{fig:boundvariance}
\end{figure}

\section{A Special Case: The Gumbel Distribution} \label{app:specialcase}
Given the analysis provided in \cref{sec:analytic}, where we illustrate that both the RQR and RQR-W objectives can narrow for coverage events and widen for miscoverage events for a single example, we might wish to better understand how they differ across a full data distribution (i.e. considering a mixture of coverage and miscoverage events). In \cref{sec:experiments} we demonstrate empirically that RQR-W does indeed find narrower solutions in aggregate, but investigating this analytically is typically not tractable. In this section, we provide the special case of a target that follows the Gumbel distribution in which we can derive a solution in closed form. Specifically, we compare the optimal solutions found by RQR and RQR-W and show that while both solutions achieve target coverage, the latter obtains narrower aggregate intervals.

We consider the Gumbel distribution $\mathcal{G}$ (with probability density function (PDF): $x\rightarrow e^{-(x+e^{-x})}$ and cumulative distribution function (CDF):  $x\rightarrow e^{-e^{(-x)}}$) as this distribution has a closed form inverse CDF: $Q : p \rightarrow -ln(-ln(p))$ which makes the calculus tractable.

Given the random variable $Y \sim \mathcal{G}$, for a coverage $\alpha$, the subset of valid intervals $(\mu_l,\mu_u)$ is characterised by the following equality $P(\mu_l \leq Y <= \mu_u) = \alpha.$

Thus, in this subset, we can express $\mu_u$ with respect to $\mu_l$ : 
\begin{align*}
    &CDF(\mu_u) - CDF(\mu_l) = \int_{\mu_l}^{\mu_u}yPdy = \alpha \\
    \implies & \mu_u = Q(\alpha + CDF(\mu_l)) \\
    \implies & \mu_u = -ln(-ln(\alpha + e^{-e^{-\mu_l}}))
\end{align*}

Hence, on the subset of valid intervals, our 2-dimensional optimization problem (finding $\mu_u$ and $\mu_l$) becomes a 1-dimensional problem (finding $\mu_l$). Moreover, we can express the interval width $W = \mu_u - \mu_l$ as a function of $\mu_l$ such that $W : \mu_l \rightarrow -ln(-ln(\alpha + e^{-e^{-\mu_l}})) -\mu_l$ and both losses as a function of $\mu_l$ in a similar manner ($\mathcal{L}^\text{RQR}(\mu_l,\mu_u) = \mathcal{L}^\text{RQR}(\mu_l)$ and $\mathcal{L}^\text{RQR-W}(\mu_l,\mu_u) = \mathcal{L}^\text{RQR-W}(\mu_l)$).

Finally, we want to show that the width of the interval found by optimizing $\mathcal{L}^\text{RQR-W}$ is less than the one obtained by optimizing $\mathcal{L}^\text{RQR}$. Thus, in Figure \ref{fig:compare}, for different values of $\mu_l$, we plot the expected value of $\mathcal{L}^\text{RQR}$, $\mathcal{L}^\text{RQR-W}$, and the respective widths of the resulting intervals (i.e. having $\mu_l$ as a lower bound).

\begin{figure}[h]
    \centering
    \includegraphics[width= 0.7\linewidth]{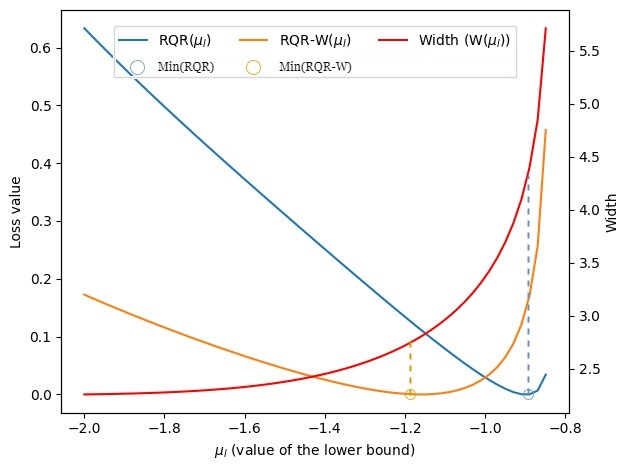}
    \caption{\textbf{A closed form solution for the Gumbel distribution.} Comparison of the minimum reached by the one-dimensional losses $ \mathcal{L}^\text{RQR}(\mu_l)$ and $\mathcal{L}^\text{RQR-W}(\mu_l)$. Note that the dashed lines visually map each loss functions minimum loss to their corresponding interval width.}
    \label{fig:compare}
\end{figure}

We name $\mu_l^{RQR}$ (respectively $\mu_l^{RQR-W}$), the lower bound that characterized the interval obtained by optimizing for the RQR objective (respectively the RQR-W objective), i.e  $\mu_l^{RQR}$ such that $\mu_l^{RQR} = \argmin_{\mu_l}\mathcal{L}^\text{RQR}(\mu_l)$.

In Figure \ref{fig:compare}, we observe that $\mu_l^{RQR} > \mu_l^{RQR-W}$ as $\mu_l^{RQR} \approx -0.9$ and $\mu_l^{RQR} \approx -1.2$ . Given the increasing width of the interval with respect to $\mu_l$ (red curve), we conclude that $W(\mu_l^{RQR}) > W(\mu_l^{RQR-W})$. \textbf{Hence, the interval found by optimizing RQR is wider than the interval found by optimizing RQR-W}. This demonstrates that including the additional regularization term in RQR-W does indeed find a solution with narrower intervals in aggregate.

\section{Additional Results}

\subsection{Complete RQR-W Results} \label{sec:RQRWcomplete}
In \Cref{tab:widthresultsfull} we present the complete results across the 12 datasets evaluating RQR-W against standard baselines (as summarized in \Cref{tab:widthresults}). For each dataset, we report both the mean coverage obtained (where the target was 90\%) and the mean interval width over 10 runs. One standard error is included in parentheses. Of course, shorter intervals are only desirable when the target level of coverage is maintained. Therefore, we exclude results that fail to achieve coverage which we indicate with a \cancel{strikethrough}. In line with previous work of \citet{tagasovska2019single}, we consider coverage to be met if the empirical coverage lies within 2.5\% of the target level $\alpha$ (after accounting for uncertainty). While obtaining empirical coverage that is greater than the desired coverage level may often be less harmful than obtaining less than the desired coverage level, at a minimum this still reflects an inefficiency. Additionally, there are many applications in which we are primarily interested in the miscoverage cases (e.g. extreme events) and, therefore, \emph{overcoverage} may be problematic in addition to being inefficient. Thus, this protocol symmetrically discards intervals that undercover and overcover by a certain margin.

\begin{table*}[h]
\caption{\small \textbf{Benchmarking RQR-W}. We evaluate each method across 12 datasets where the first row reports the marginal coverage obtained and the second row reports interval width as measured using MPIW. All results report the test set mean over 10 runs ($\pm$ a standard error). We \cancel{strikethrough} results where desired coverage is not achieved (see text for details). N/A$^*$ indicates a case of excessively wide intervals due to non-covered examples not being penalized for width -- exact values 1077.42 (1073.82).}
\label{tab:widthresultsfull}
\centering
\begin{tabular}{ccccccc} 
\toprule
Dataset & Ours & QR & SQR-C & SQR-N & IR & Dist. Hist. \\
\hline \\[-3mm]
\multirow{2}{0.098\linewidth}{\hspace{0pt}miami} & 88.89 (0.32) & 90.50 (0.57) & 87.36 (0.46) & \cancel{85.97} (0.48) & 90.26 (0.49) & \multirow{2}{*}{\parbox[b]{2cm}{\includegraphics[width=\linewidth, height=0.7cm]{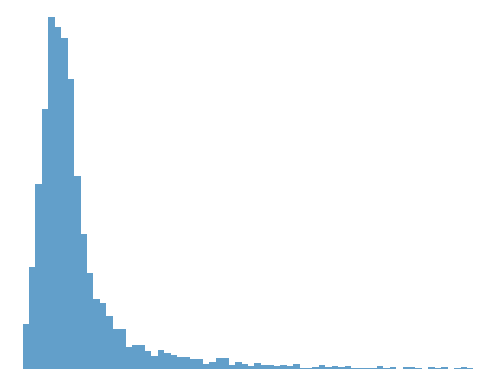}}} \\
 & \textbf{0.51} (0.01) & \textbf{0.51} (0.01) & 1.90 (0.05) & 1.52 (0.03) & 1.74 (0.12) &  \\[1.5mm]
\multirow{2}{0.098\linewidth}{\hspace{0pt}kin8nm} & 89.27 (0.40) & 91.89 (0.27) & 87.16 (0.49) & \cancel{85.53} (0.53) & 89.52 (0.46) & \multirow{2}{*}{\parbox[b]{2cm}{\includegraphics[width=\linewidth, height=0.7cm]{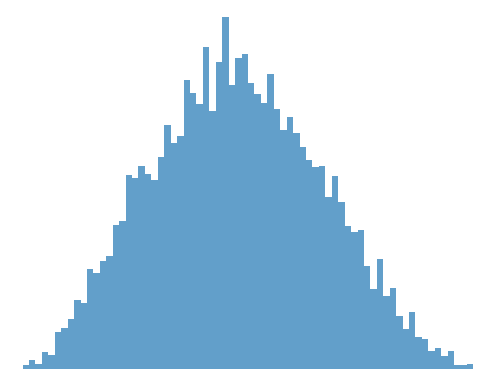}}} \\
 & \textbf{0.42} (0.01) & 0.50 (0.01) & 1.14 (0.01) & 1.10 (0.01) & 1.61 (0.12) &  \\[1.5mm]
\multirow{2}{0.098\linewidth}{\hspace{0pt}protein} & 88.47 (0.30) & 90.40 (0.23) & 89.01 (0.20) & \cancel{86.01} (0.32) & 90.38 (0.30) & \multirow{2}{*}{\parbox[b]{2cm}{\includegraphics[width=\linewidth, height=0.7cm]{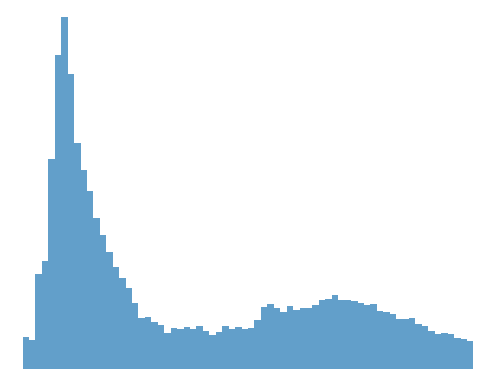}}} \\
 & \textbf{1.50} (0.01) & 1.61 (0.01) & 2.19 (0.02) & 2.05 (0.01) & 1.96 (0.09) &  \\[1.5mm]
\multirow{2}{0.098\linewidth}{\hspace{0pt}yacht} & 90.32 (0.42) & 91.45 (1.23) & \cancel{85.97} (1.18) & \cancel{84.68} (1.69) & 90.32 (0.96) &\multirow{2}{*}{\parbox[b]{2cm}{\includegraphics[width=\linewidth, height=0.7cm]{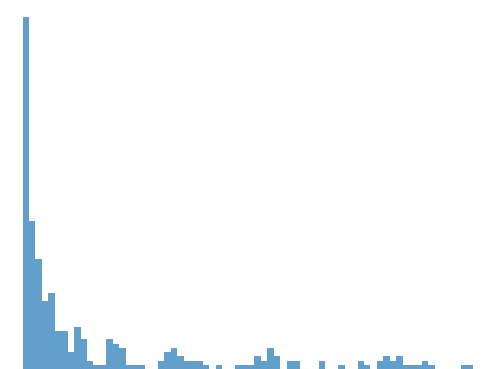}}} \\
 & \textbf{0.33} (0.02) & \textbf{0.32} (0.02) & 3.56 (0.38) & 2.91 (0.26) & 1.66 (0.29) &  \\[1.5mm]
\multirow{2}{0.098\linewidth}{\hspace{0pt}wine} & 90.16 (0.37) & 89.97 (0.36) & 91.94 (0.67) & 88.19 (0.97) & 88.03 (0.92) & \multirow{2}{*}{\parbox[b]{2cm}{\includegraphics[width=\linewidth, height=0.7cm]{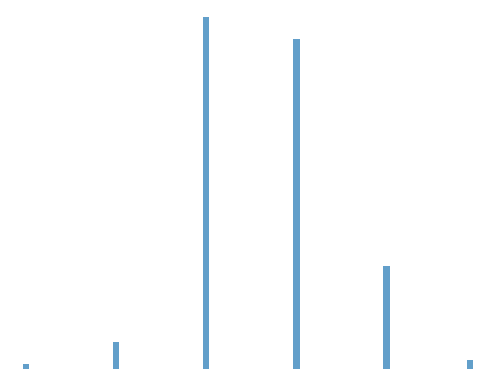}}} \\
 & 0.33 (0.01) & \textbf{0.28} (0.00) & 0.49 (0.03) & 0.48 (0.03) & 1.18 (0.12) &  \\[1.5mm]
\multirow{2}{0.098\linewidth}{\hspace{0pt}power} & 89.34 (0.55) & 89.92 (0.54) & 91.81 (1.34) & 89.11 (1.40) & \cancel{62.47} (13.63) & \multirow{2}{*}{\parbox[b]{2cm}{\includegraphics[width=\linewidth, height=0.7cm]{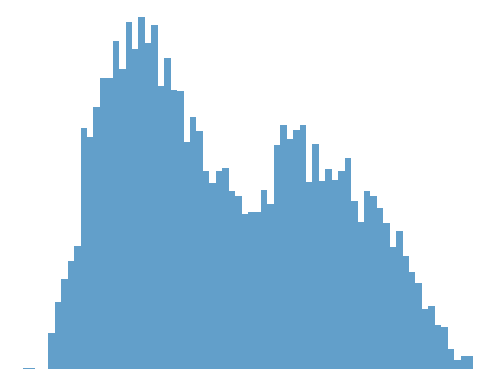}}} \\
 & \textbf{0.07} (0.01) & \textbf{0.06} (0.01) & 0.15 (0.03) & 0.15 (0.03) & 1.14 (0.23) &  \\[1.5mm]
\multirow{2}{0.098\linewidth}{\hspace{0pt}boston} & 88.14 (0.66) & 86.67 (1.06) & \cancel{82.55} (1.48) & \cancel{81.18} (1.58) & 90.29 (0.77) & \multirow{2}{*}{\parbox[b]{2cm}{\includegraphics[width=\linewidth, height=0.7cm]{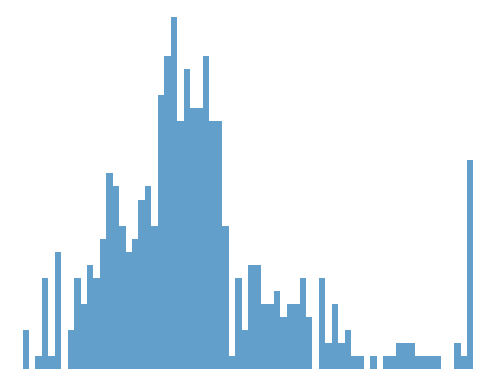}}} \\
 & \textbf{0.40} (0.02) & \textbf{0.38} (0.01) & 1.09 (0.03) & 1.01 (0.02) & 1.41 (0.15) &  \\[1.5mm]
\multirow{2}{0.098\linewidth}{\hspace{0pt}energy} & 89.68 (0.70) & 93.05 (0.88) & 89.42 (0.99) & \cancel{86.30} (0.99) & 89.42 (0.76) & \multirow{2}{*}{\parbox[b]{2cm}{\includegraphics[width=\linewidth, height=0.7cm]{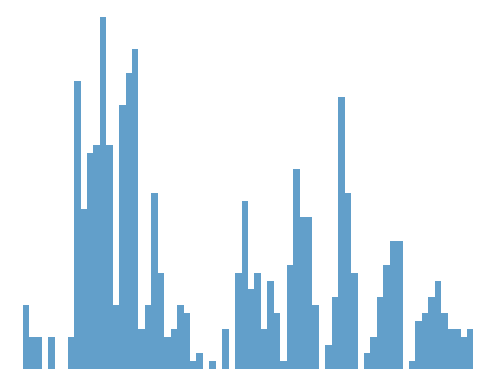}}} \\
 & \textbf{0.23} (0.01) & 0.29 (0.01) & 1.31 (0.01) & 1.23 (0.01) & 1.39 (0.17) &  \\[1.5mm]
\multirow{2}{0.098\linewidth}{\hspace{0pt}sulfur} & 88.87 (0.20) & 88.30 (0.34) & 87.83 (0.47) & 87.40 (0.50) & 89.41 (0.53) & \multirow{2}{*}{\parbox[b]{2cm}{\includegraphics[width=\linewidth, height=0.7cm]{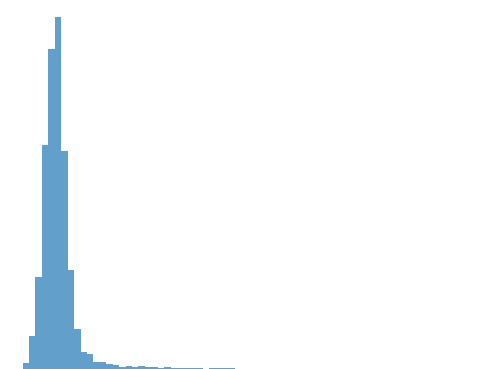}}} \\
 & \textbf{1.02} (0.01) & 1.05 (0.01) & 1.09 (0.03) & \textbf{1.02} (0.02) & 1.36 (0.07) &  \\[1.5mm]
\multirow{2}{0.098\linewidth}{\hspace{0pt}cpu\_act} & 89.14 (0.37) & 88.94 (0.48) & 90.85 (0.75) & 86.73 (1.10) & 90.32 (0.62) & \multirow{2}{*}{\parbox[b]{2cm}{\includegraphics[width=\linewidth, height=0.7cm]{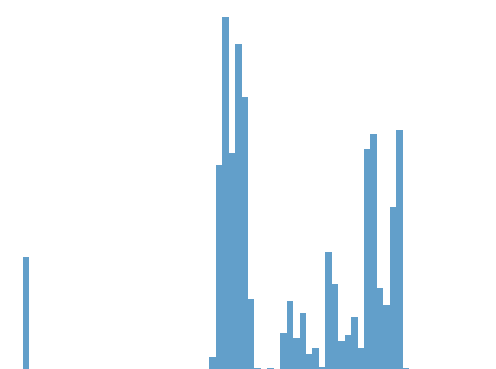}}} \\
 & 0.44 (0.00) & \textbf{0.41} (0.01) & 0.78 (0.01) & 0.76 (0.02) & N/A$^*$ (N/A$^*$) &  \\[1.5mm]
\multirow{2}{0.098\linewidth}{\hspace{0pt}concrete} & 88.74 (0.66) & 87.77 (1.09) & \cancel{86.02} (1.29) & \cancel{85.29} (1.39) & 89.37 (0.68) & \multirow{2}{*}{\parbox[b]{2cm}{\includegraphics[width=\linewidth, height=0.7cm]{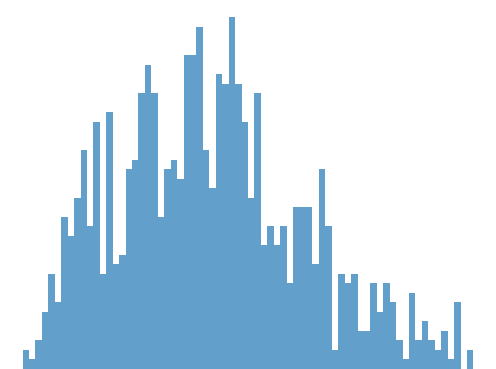}}} \\
 & \textbf{0.47} (0.03) & \textbf{0.44} (0.01) & 1.39 (0.03) & 1.33 (0.02) & 1.28 (0.16) &  \\[1.5mm]
\multirow{2}{0.098\linewidth}{\hspace{0pt}naval} & 88.81 (0.22) & 88.34 (0.25) & 90.70 (1.68) & 88.86 (1.77) & 91.16 (0.35) & \multirow{2}{*}{\parbox[b]{2cm}{\includegraphics[width=\linewidth, height=0.7cm]{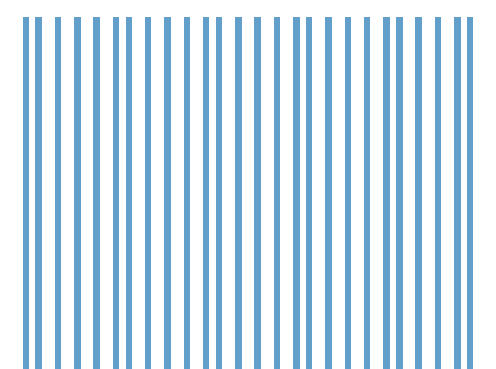}}} \\
 & 0.03 (0.00) & \textbf{0.02} (0.00) & 0.26 (0.04) & 0.26 (0.04) & 2.93 (0.18) &  \\
\bottomrule
\end{tabular}
\end{table*}

\clearpage

\subsection{Evaluating Different Coverage Levels}\label{sec:otheralpha}
Lower coverage intervals (i.e. narrower intervals) can lead to greater instability in the estimates. Thus, it is valuable to evaluate our method in the low coverage probability regime to ensure this increased instability doesn't have a disproportionally large effect on our method. We empirically analyze the relative performance of our method for 70\%, 50\%, and 30\% coverage levels in \Cref{tab:low_coverage_70}, \Cref{tab:low_coverage_50} and \Cref{tab:low_coverage_30}. Overall these results are consistent with the 90\% case establishing that RQR-W and QR also obtain narrower intervals when we consider alternative coverage levels.

\begin{table}[h]
\centering
\caption{\textbf{Comparison on 70\% targeted coverage.} For the 9 datasets, we display interval width and the coverage level achieved in parentheses ± a standard error for both. All results report the test set mean over 10 runs.}
\label{tab:low_coverage_70}
\begin{tblr}{
  width = \linewidth,
  colspec = {Q[125]Q[160]Q[160]Q[160]Q[160]Q[173]},
  cells = {c},
  cell{2}{1} = {r=2}{},
  cell{4}{1} = {r=2}{},
  cell{6}{1} = {r=2}{},
  cell{8}{1} = {r=2}{},
  cell{10}{1} = {r=2}{},
  cell{12}{1} = {r=2}{},
  cell{14}{1} = {r=2}{},
  cell{16}{1} = {r=2}{},
  cell{18}{1} = {r=2}{},
  hline{1,20} = {-}{0.08em},
  hline{2,4,6,8,10,12,14,16,18} = {-}{},
} 
Dataset & Ours & QR & SQR-C & SQR-N & IR\\
kin8nm & 69.27 (0.30) & 70.32 (0.49) & 69.34 (0.38) & 62.87 (1.23) & 70.71 (0.35)\\
 & 0.29 (0.02) & 0.28 (0.02) & 0.81 (0.00) & 0.73 (0.02) & 1.09 (0.10)\\
protein & 69.02 (0.13) & 69.85 (0.33) & 70.41 (0.51) & 60.93 (1.77) & 69.46 (0.30)\\
 & 0.93 (0.01) & 1.03 (0.01) & 1.77 (0.02) & 1.36 (0.03) & 1.09 (0.02)\\
yacht & 69.52 (0.61) & 71.94 (2.42) & 72.58 (2.32) & 60.16 (5.07) & 72.58 (0.72)\\
 & 0.18 (0.02) & 0.17 (0.01) & 1.68 (0.21) & 1.00 (0.06) & 5.84 (4.92)\\
wine & 70.47 (0.59) & 70.47 (1.65) & 78.69 (1.30) & 73.66 (3.10) & 69.59 (0.26)\\
 & 0.22 (0.02) & 0.18 (0.00) & 0.30 (0.02) & 0.25 (0.03) & 1.17 (0.06)\\
power & 68.88 (0.14) & 70.02 (0.85) & 69.48 (1.10) & 58.30 (1.22) & 48.90 (10.68)\\
 & 0.03 (0.00) & 0.02 (0.00) & 0.08 (0.00) & 0.07 (0.00) & 0.89 (0.17)\\
boston & 70.39 (0.48) & 63.33 (1.96) & 67.75 (2.02) & 64.80 (2.67) & 70.39 (0.96)\\
 & 0.30 (0.05) & 0.20 (0.01) & 0.70 (0.02) & 0.60 (0.03) & 0.92 (0.16)\\
energy & 69.74 (0.41) & 68.83 (1.45) & 70.26 (2.11) & 59.35 (1.71) & 70.97 (0.72)\\
 & 0.15 (0.02) & 0.11 (0.01) & 0.98 (0.01) & 0.85 (0.03) & 0.89 (0.13)\\
concrete & 66.55 (0.73) & 65.15 (0.74) & 68.45 (1.41) & 63.20 (1.88) & 69.81 (0.52)\\
 & 0.27 (0.01) & 0.23 (0.00) & 0.97 (0.02) & 0.89 (0.02) & 1.03 (0.11)\\
naval & 69.31 (0.35) & 69.29 (0.49) & 70.32 (0.51) & 58.09 (2.73) & 69.49 (0.22)\\
 & 0.02 (0.00) & 0.01 (0.00) & 0.14 (0.03) & 0.13 (0.03) & 1.22 (0.13)
\end{tblr}
\end{table}

\begin{table}[h]
\centering
\caption{\textbf{Comparison on 50\% targeted coverage.} For the 9 datasets, we display interval width and the coverage level achieved in parentheses ± a standard error for both. All results report the test set mean over 10 runs.}
\label{tab:low_coverage_50}
\label{tab:ablation_emrqa}
\begin{tblr}{
  width = \linewidth,
  colspec = {Q[127]Q[162]Q[162]Q[162]Q[162]Q[162]},
  cells = {c},
  cell{2}{1} = {r=2}{},
  cell{4}{1} = {r=2}{},
  cell{6}{1} = {r=2}{},
  cell{8}{1} = {r=2}{},
  cell{10}{1} = {r=2}{},
  cell{12}{1} = {r=2}{},
  cell{14}{1} = {r=2}{},
  cell{16}{1} = {r=2}{},
  cell{18}{1} = {r=2}{},
  hline{1,20} = {-}{0.08em},
  hline{2,4,6,8,10,12,14,16,18} = {-}{},
}
Dataset & Ours & QR & SQR-C & SQR-N & IR\\
kin8nm & 49.55 (0.30) & 49.34 (0.70) & 49.84 (0.64) & 39.23 (3.63) & 51.26 (0.32)\\
 & 0.18 (0.01) & 0.16 (0.01) & 0.53 (0.01) & 0.43 (0.04) & 1.10 (0.35)\\
protein & 49.90 (0.24) & 49.80 (0.40) & 49.89 (0.40) & 43.81 (2.70) & 49.45 (0.28)\\
 & 0.63 (0.01) & 0.64 (0.01) & 1.35 (0.01) & 0.57 (0.06) & 0.56 (0.01)\\
yacht & 49.35 (0.73) & 54.19 (2.75) & 52.74 (3.66) & 50.32 (3.98) & 50.48 (1.42)\\
 & 0.09 (0.01) & 0.10 (0.01) & 1.23 (0.04) & 0.47 (0.08) & 0.36 (0.08)\\
wine & 49.53 (0.24) & 49.16 (0.97) & 56.16 (2.97) & 39.22 (3.03) & 48.53 (0.73)\\
 & 0.15 (0.01) & 0.09 (0.00) & 0.23 (0.02) & 0.15 (0.03) & 0.87 (0.08)\\
power & 49.68 (0.16) & 48.91 (0.52) & 49.75 (0.93) & 42.40 (1.78) & 34.97 (7.64)\\
 & 0.02 (0.00) & 0.01 (0.00) & 0.06 (0.00) & 0.04 (0.00) & 0.82 (0.26)\\
boston & 48.73 (0.94) & 40.69 (1.42) & 47.55 (1.89) & 31.86 (3.28) & 49.51 (0.44)\\
 & 0.23 (0.06) & 0.11 (0.00) & 0.36 (0.02) & 0.28 (0.03) & 0.89 (0.15)\\
energy & 49.81 (0.36) & 51.36 (1.76) & 53.05 (1.32) & 42.86 (2.59) & 50.00 (1.51)\\
 & 0.08 (0.01) & 0.06 (0.00) & 0.78 (0.01) & 0.46 (0.05) & 0.68 (0.09)\\
concrete & 48.59 (1.14) & 42.52 (0.84) & 51.75 (1.50) & 39.85 (3.29) & 50.15 (0.75)\\
 & 0.23 (0.05) & 0.13 (0.00) & 0.66 (0.01) & 0.55 (0.03) & 0.89 (0.09)\\
naval & 49.45 (0.14) & 49.69 (0.66) & 49.47 (1.50) & 37.00 (1.89) & 49.95 (0.30)\\
 & 0.01 (0.00) & 0.01 (0.00) & 0.06 (0.02) & 0.06 (0.02) & 1.67 (0.56)
\end{tblr}
\end{table}

\begin{table}[h]
\centering
\caption{\textbf{Comparison on 30\% targeted coverage.} For the 9 datasets, we display interval width and the coverage level achieved in parentheses ± a standard error for both. All results report the test set mean over 10 runs.}
\label{tab:low_coverage_30}
\begin{tblr}{
  width = \linewidth,
  colspec = {Q[127]Q[162]Q[162]Q[162]Q[162]Q[162]},
  cells = {c},
  cell{2}{1} = {r=2}{},
  cell{4}{1} = {r=2}{},
  cell{6}{1} = {r=2}{},
  cell{8}{1} = {r=2}{},
  cell{10}{1} = {r=2}{},
  cell{12}{1} = {r=2}{},
  cell{14}{1} = {r=2}{},
  cell{16}{1} = {r=2}{},
  cell{18}{1} = {r=2}{},
  hline{1,20} = {-}{0.08em},
  hline{2,4,6,8,10,12,14,16,18} = {-}{},
}
Dataset & Ours & QR & SQR-C & SQR-N & IR\\
kin8nm & 29.85 (0.18) & 29.21 (0.58) & 30.58 (0.53) & 20.08 (2.64) & 30.34 (0.35)\\
 & 0.12 (0.02) & 0.08 (0.00) & 0.31 (0.01) & 0.24 (0.03) & 5.31 (4.95)\\
protein & 29.82 (0.13) & 29.79 (0.14) & 29.56 (0.22) & 18.37 (1.94) & 29.77 (0.17)\\
 & 0.40 (0.02) & 0.31 (0.01) & 0.88 (0.01) & 0.13 (0.02) & 0.31 (0.03)\\
yacht & 30.48 (0.38) & 26.94 (1.72) & 31.13 (3.21) & 18.71 (3.94) & 30.97 (1.31)\\
 & 0.05 (0.01) & 0.04 (0.00) & 0.60 (0.04) & 0.16 (0.05) & 0.19 (0.02)\\
wine & 29.81 (0.29) & 29.56 (1.04) & 30.41 (2.20) & 26.09 (2.90) & 30.13 (0.41)\\
 & 0.08 (0.01) & 0.02 (0.00) & 0.16 (0.01) & 0.10 (0.02) & 0.55 (0.07)\\
power & 29.71 (0.16) & 28.80 (0.40) & 29.76 (0.81) & 18.79 (2.21) & 20.84 (4.55)\\
 & 0.01 (0.00) & 0.01 (0.00) & 0.06 (0.01) & 0.04 (0.01) & 0.77 (0.41)\\
boston & 28.82 (0.66) & 20.78 (1.52) & 28.53 (2.70) & 15.49 (2.39) & 30.49 (0.75)\\
 & 0.07 (0.00) & 0.05 (0.00) & 0.18 (0.02) & 0.14 (0.02) & 0.48 (0.09)\\
energy & 30.19 (0.22) & 31.17 (1.80) & 31.43 (1.32) & 16.17 (2.59) & 30.52 (1.12)\\
 & 0.04 (0.00) & 0.03 (0.00) & 0.54 (0.01) & 0.14 (0.02) & 4.65 (4.30)\\
concrete & 29.17 (0.57) & 23.25 (1.17) & 31.75 (1.99) & 19.76 (3.71) & 29.56 (0.43)\\
 & 0.10 (0.01) & 0.07 (0.00) & 0.37 (0.01) & 0.26 (0.04) & 0.58 (0.06)\\
naval & 30.04 (0.19) & 29.41 (0.40) & 30.20 (0.66) & 16.38 (1.52) & 30.06 (0.11)\\
 & 0.05 (0.04) & 0.00 (0.00) & 0.06 (0.02) & 0.05 (0.01) & 0.84 (0.13)
\end{tblr}
\end{table}


\clearpage

\subsection{Non-Smoothed kin8nm Intervals}
\begin{figure}[h]
    \centering
    \includegraphics[width=\textwidth]{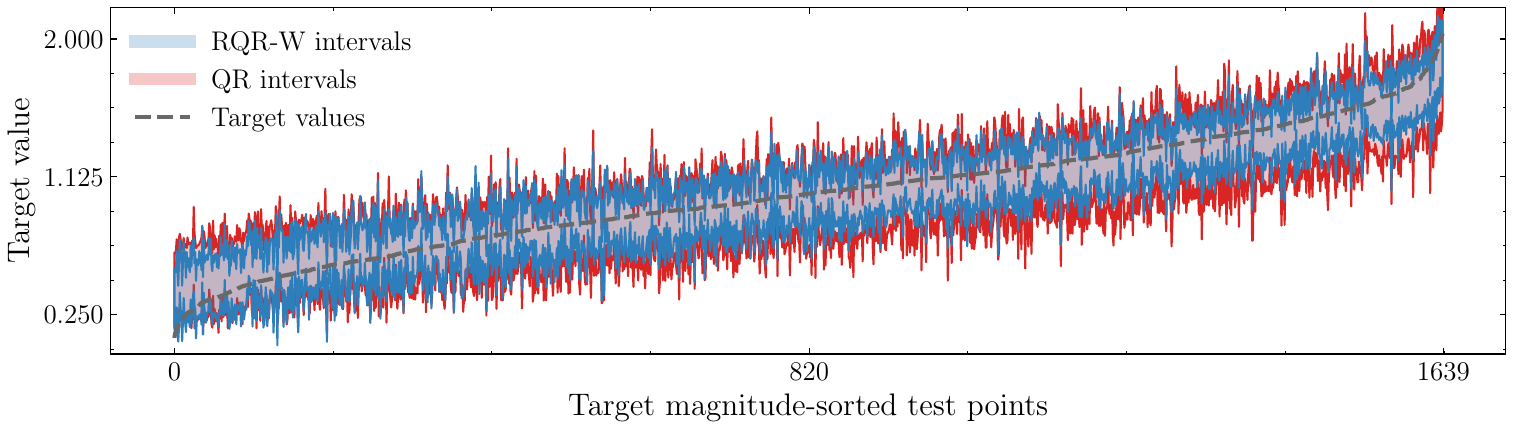}
    \caption{\textbf{Non-smoothed kin8nm intervals.} \Cref{fig:kin8nm} (left) without Savitzky–Golay filter applied. In this version it is apparent that micoverage events do occur as we would expect while the difference in interval width is less legible.}
    \label{fig:unsmoothkin8nm}
\end{figure}

\end{document}